\documentclass[twoside,11pt]{article}

\usepackage{rotating}
\usepackage{pdflscape}

\usepackage{jmlr2e} 
\usepackage{hyperref} 

\usepackage{amsmath, amsfonts, mathabx}
\usepackage{bm, bbm, dsfont}
\usepackage{graphicx, subfig}
\usepackage{blkarray, alltt}

\usepackage{tablefootnote}
\usepackage{booktabs}
\usepackage{multirow}
\usepackage{threeparttable}

\usepackage{algorithm}
\usepackage[noend]{algpseudocode}

\usepackage{color}

\newcommand{\btheta}{\boldsymbol{\theta}}

\newcommand{\iter}{{l}}

\newcommand{\MG}{{\cal{MG}}}
\newcommand{\mg}{{\tt{mg}}}

\newcommand{\ma}{{${\tt MA}$}}
\newcommand{\mar}{{${\tt MA_{r}}$}}
\newcommand{\maunm}{{${\tt MA_{un}}$}}

\newcommand{\proj}[1]{{\Pi}}
\newcommand{\sel}[1]{{\sigma}}

\newcommand{\cut}[1]{}
\newcommand{\cutfull}[1]{}

\newcommand{\commentresolved}[1]{}

\newcommand{\ie}{{\it i.e.}} 
\newcommand{\eg}{{\it e.g.}} 

\cut{

}

\newcommand{\ismatched}{{\tt is\_matched}}

\newcommand{\BasicExactMatch}{{\tt BasicExactMatch}}

\newcommand{\MQ}{{\tt MQ}}

\newcommand{\BF}{{\tt BF}}
\newcommand{\PE}{{\tt PE}}

\newcommand{\x}{\mathbf{x} }

\newcommand{\MGS}{\mathcal{MG}(\x, \btheta, \mathcal{S}^{ma})}

\newcommand{\nxt}{n_t(\x, \btheta, \mathcal{S}^{ma})}
\newcommand{\nxc}{n_c(\x, \btheta, \mathcal{S}^{ma})}

\newcommand{\com}[1]{&&\mbox{(#1)}}

\newcommand{\vneq}[2]{\mathbf{1}_{(#1 \neq #2)}}
\newcommand{\bw}{\mathbf{w}}
\newcommand{\bone}{\mathbf{1}}

\usepackage{lastpage}
\jmlrheading{21}{2020}{1-\pageref{LastPage}}{10/19; Revised
12/20}{12/20}{19-853}{Tianyu Wang, Marco Morucci, M. Usaid Awan, Yameng Liu, Sudeepa Roy, Cynthia Rudin, Alexander Volfovsky}
\ShortHeadings{FLAME: Fast Large-scale Almost Matching Exactly}{Wang, Morucci, Awan, Liu, Roy, Rudin, Volfovsky} 

\begin{document}

\title{FLAME: A Fast Large-scale Almost Matching Exactly Approach to Causal Inference} 

\author{\name Tianyu Wang \email tianyu@cs.duke.edu \AND \name Marco Morucci \email marco.morucci@duke.edu \AND \name M. Usaid Awan \email muhammad.awan@duke.edu \AND \name Yameng Liu \email yameng.liu@duke.edu \AND \name Sudeepa Roy \email sudeepa@cs.duke.edu \AND \name Cynthia Rudin \email cynthia@cs.duke.edu \AND \name Alexander Volfovsky \email alexander.volfovsky@duke.edu \\ \addr Duke University } 

\editor{Russ Greiner}


\maketitle

\begin{abstract}
 
  A classical problem in causal inference is that of \textit{matching}, where treatment units need to be matched to control units based on covariate information. In this work, we propose a method that computes high quality almost-exact matches for high-dimensional categorical datasets. This method, called \emph{FLAME (Fast Large-scale Almost Matching Exactly)}, learns a distance metric for matching using a hold-out training data set.
  In order to perform matching efficiently for large datasets, FLAME leverages techniques that are natural for query processing in the area of database management, and two implementations of FLAME are provided: the first uses SQL queries and the second uses bit-vector techniques. The algorithm starts by constructing matches of the highest quality (exact matches on all covariates), and successively eliminates variables in order to match exactly on as many variables as possible, while still maintaining interpretable high-quality matches and balance between treatment and control groups. We leverage these high quality matches to estimate conditional average treatment effects (CATEs). Our experiments show that FLAME scales to huge datasets with millions of observations where existing state-of-the-art methods fail, and that it achieves significantly better performance than other matching methods.  
\end{abstract}

\section{Introduction}\label{sec:introduction}

Questions of robust causal inference, beyond simple correlations or model-based predictions, are practically unavoidable in 
health, medicine, or social studies. Causal inference goes beyond simpler correlation, association, or model-based predictive analysis as it attempts to estimate the causal effects of a certain intervention. 

Much of the available data in the clinical and social sciences is observational. 
In such situations individuals may be selected in a way that treatment assignments depend on outcomes: people who benefit more from pain relievers tend to take them more often, individuals who are likely to succeed in higher education are more likely to enroll in it, and so forth. 
Estimating causal effects in an observational setting becomes a problem of representing the available data as if it were collected from a randomized experiment in which individuals are assigned to treatment independently of their potential outcomes. 

A natural approach to observational studies is matching treated and control units such that underlying background covariates are balanced \citep{chapin1947experimental,greenwood1945experimental}. Under the assumption of unconfoundedness, such matches allow the practitioner to estimate causal effects.
If we can match units exactly (or almost exactly) on raw covariate values, the practitioner can further interpret causal estimates within matched groups as conditional average treatment effects. Exact matching increases the interpretability and usefulness of causal analyses in several ways: It is a tool for granular causal analysis that can provide crucial information on who benefits from treatment most, where resources should be spent for future treatments, and why some individuals are treated while others were not. It can provide explanations for treatment effect estimates in a way that pure modeling methods (that do not use matching) cannot. 
It helps determine what type of additional data must be collected to control for confounding. Interpretability is important: a recent empirical study \citep{dorie2017automated} suggests that CATEs estimated by manually calibrated methods can be better than CATEs estimated by black box models, even though the latter can achieve good performance if they are implemented with careful consideration of the specific application \citep{hitsch2018heterogeneous}. These statements are in agreement with other general observations that interpretable models do not necessarily lead to sacrifices in accuracy because they allow humans to troubleshoot more effectively \citep{Rudin19}. 

In this work, we propose an approach to matching under the potential outcomes framework with a binary treatment, for datasets with a possibly large number of discrete covariates. Our method (FLAME -- Fast, Large-scale, Almost Matching Exactly) creates matches that are \textit{almost-exact}, meaning that it tries to match treatment and control units exactly on important covariates. The main benefits of FLAME are:
\begin{itemize}
\item It \textit{learns} a weighted Hamming distance for matching based on a hold-out training set (rather than using a pre-specified distance). 
In particular, it successively drops covariates, but always retains enough covariates for high quality CATE estimation and balance between treatment and control groups.
\item It lends itself naturally to large datasets, even those that are too large to fit in memory.  FLAME has two implementations. One implementation uses bit vectors, and is extremely fast for data that has been preprocessed and fits in memory. The other implementation leverages database management systems  \citep[e.g.,][]{postgresql}, and in particular, highly optimized built-in SQL \emph{group-by} operators that can operate directly on the database, and can handle data sets that are too large to fit in memory. The use of database systems makes the matching algorithm suitable for parallel executions. The database implementation is specialized to be efficient, and only a few lines of carefully constructed SQL code are needed to perform matching. 
\end{itemize}




FLAME improves over current coarsened exact matching, mixed integer programming matching, and network flow methods in that it does not introduce a distance metric on the covariate space \textit{a priori}, instead, it \textit{learns} the distance metric. For this reason, it does not suffer when irrelevant variables are introduced. It further improves over regression, propensity and their variable selection methods by not forcing a model form on either the treatment or the outcome, instead it matches on covariates directly (nonparametrically). Unlike other nonparametric methods, like black box machine learning approaches, FLAME is interpretable.

By successively dropping covariates to permit matches, FLAME increases bias in order to make predictions of conditional average treatment effect. 
We can calculate FLAME's bias directly in some cases, and its bias depends directly on how important the dropped covariates are to predicting the output. If only the irrelevant covariates are dropped, the estimates are unbiased. In that case, FLAME's estimates and the gold standard estimates of exact matches are identical. 

We discuss FLAME's relationship to prior work  (Section~\ref{sec:related}) and introduce FLAME's framework (Section \ref{sec:flame-framework}). We present  FLAME's algorithm and implementation (Section \ref{sec:algorithm-implementation}), give theoretical suggestions of FLAME's statistical bias (Section \ref{sec:theory}), and provide experiments (Section \ref{sec:experiments}). 


\section{Relationship to Prior Work}\label{sec:related}

Early approaches to matching considered exact matching on covariates but quickly ran into issues of insufficient sample size when the number of covariates was even moderately large. In high dimensions, there simply are not enough treated and control units with exactly the same values of all the covariates. 
In the 1970's and 1980's, a large literature on different dimension reduction approaches to matching was developed \citep[e.g.,][]{rubin1973matching,rubin1973use,rubin1976multivariate,cochran1973controlling} with the extreme being work on propensity score matching, which was later extended to work on penalized regression approaches that leverage propensity \citep{schneeweiss2009high,rassen2012using,belloni2014inference,farrell2015robust}.
The problem with propensity score methods (and other parametric methods) is that they require a proper specification of a model for the treatment assignment probability, which can never be provided in practice. With doubly robust methods, analysts can specify both a model for the outcome and for the propensity score, with only one of the two having to be correct in order to obtain unbiased estimates; however, there is no reason to presume that either of these models would always be specified correctly in practice. These methods can be augmented with variable selection procedures at either or both steps of the estimation procedure, which can improve estimation but leads to further questions of uncertainty quantification \citep{brookhart2006variable, hahn2004functional}. 

Further, in these dimension reduction techniques \citep{schneeweiss2009high}, along with more modern neural network-based dimension-reduction techniques, the reduced feature spaces cease to be interpretable with respect to the original input variables, and the matches are no longer meaningful. While these methods are uninterpretable, it has recently been shown \citep{farrell2018deep} that neural networks can achieve asymptotically consistent treatment effect estimates.

An important type of technique, coarsened exact matching (CEM), creates matches that attempt to preserve more covariate information.  
It bypasses the need to fit complicated propensity score models by coarsening or discretizing covariates in such a way that the newly constructed covariates allow for exact matching \citep{iacus2011causal,iacus2011multivariate}.  This approach is appealing when there are many continuous covariates that are naturally amenable to binning \citep{cattaneo2011efficient}. However, when most or all of the covariates are categorical, which is the case we consider in this work, \textit{a priori} coarsening becomes impossible without introducing a calculus on all of the covariates. This can be problematic in high dimensions, and tends to have the same problem as nearest neighbor techniques, which are well-known to perform poorly in high dimensions because they use manually chosen distance metrics. Further, when categorical variables have two levels, coarsening is equivalent to variable selection. In this setting, coarsened exact matching would lead to the same matches as high dimensional propensity score techniques with variable selection.

A similar problem exists with what is called optimal matching in the literature \citep{rosenbaum2016imposing}. A distance metric over variables is defined manually, which introduces a calculus on the covariates. That distance metric is used as input to a network flow problem which optimizes match quality. Despite the optimality of the solution network flow problem, the quality of the matches is questionable since, again, it relies on a manually defined distance measure. Network flow problems also cannot directly handle constraints; the user needs to manually manipulate the algorithm in order to obtain desired balance constraints \citep{zubizarreta2012using}.



The problems with network flow optimization highlight an important concern about non-exact matching methods generally, which is that they implicitly approximate the solution of a hard combinatorial optimization problem. It is possible that high quality match assignments exist, but the standard approximate methods of constructing matches, such as network flow optimization, did not find them. To handle the problem of finding suboptimal match assignments, some newer matching schemes use mixed integer programming, which is a flexible framework that can accommodate linear balance constraints \citep{zubizarreta2012using,zubizarreta2014matching,keele2014optimal,resa2016evaluation,AlamRu15,AlamRu15nonparam}.
However, these methods have two major disadvantages: first they cannot scale to large problems;
 second, they may be trying to match units on covariates that are not important in any way. 
  In the simulations later, we show how these issues with past work heavily affect practical performance of these methods, whereas FLAME, which learns interpretable distances on the covariates, does not seem to have these problems. 
  
 We do not recommend using FLAME on continuous covariates unless there is some substantive knowledge that allows the user to bin the covariates in a way that would not hinder downstream causal estimation. An example of such knowledge is that the outcome and covariates vary smoothly with the continuous variable and that the bins are sufficiently small to model that variation. While our work considers large scale problems with mainly discrete or categorical covariates, its extensions have adapted its main ideas to handle continuous covariates by learning either adaptive stretch metrics \citep{ParikhRuVo18,MaltsObsStudies} or adaptive hyperboxes (AHB) \citep{AME-AHB-2020}, both of which use distance metrics that are different than the adaptive Hamming distance that FLAME uses. FLAME can be used alongside these methods, for instance, to eliminate a large number of irrelevant categorical variables to attain a smaller dataset on which MALTS or AHB can operate on both continuous and categorical variables.

\par


\section{The FLAME Framework} \label{sec:flame-framework}
Suppose $\mathcal{S}=[X, Y, T]$ is the \emph{data table} with $n$ units. $X$ is $n\times d$, $Y$ and $T$ are $n\times 1$, where each row corresponds to a \emph{unit} (called a \emph{tuple} in  database management). The columns of $X$ correspond to  \emph{variables} (also known as\ \emph{covariates}, \emph{features}, or, \emph{attributes}).
Data must be categorical in order for exact matches to occur with non-zero probability; the match is approximate with binned real-valued data.
Treatment assignment variable $T$   takes  binary values 1 (for treatment) and 0 (for control). $Y^{(1)}$ and  $Y^{(0)}$ are the $n\times 1$ vectors of potential outcomes. 


In real-world field studies, researchers tend to collect as many covariates as possible, including irrelevant covariates.
Throughout the remaining exposition, we divide the covariates into relevant covariates: $X^{REL}$ and irrelevant covariates:  $X^{IRR}$ and make the following standard assumptions: (1) Stable Unit Treatment Value \citep{Rubin2005}; (2) overlap of support (of the treated population and the control population); (3) Strong ignorability \citep{rosenbaum1983central}: $Y^{(1)}, Y^{(0)}  \bot T | X^{REL}  $. (4) Further, for irrelevant covariates $X^{IRR}$ we assume that  $Y^{(1)}, Y^{(0)}  \bot X^{IRR},  $ and $T \bot X^{IRR} | X^{REL} $. 
 Based on our assumptions, our estimand is $\mathbb{E}[Y^{(1)}-Y^{(0)}|X^{REL}] = \mathbb{E}[Y^{(1)}-Y^{(0)}|X] $. 
 Within each matched group, we use the difference between the average outcome of the treated units and the average outcome of the control units to estimate CATE given the covariate values associated with that group. \emph{Note that the matching does not produce estimation, it produces a partition of the covariate space, based on which we can estimate CATEs}. 

FLAME learns an interpretable weighted Hamming distance on discrete data. This allows it to match on the important covariates, without necessarily matching on unimportant covariates. In the following subsection we discuss the importance of learning a distance metric.

\subsection{Why Learn a Distance Metric for Matching?}

In exact matching schemes, we would ideally match on as many covariates as possible. However, exact matching on as many covariates as possible leads to a serious issue, namely the \textit{toenail problem}. The toenail problem is where irrelevant covariates dominate the distance metric for matching. Irrelevant covariates are related neither to treatment assignment or outcome, such as length of toenails when considering heart attack outcomes. 
The researcher typically chooses a distance metric in the high dimensional covariate space (Hamming distance, or squared distance perhaps) in matching without examining outcome variables. The researcher might weigh each covariate equally in the definition of the distance metric. By this choice of distance metric, the irrelevant covariates could dominate the distance measure; adding many covariates that are random noise would make the matches closer to random matches. The matches become essentially meaningless and resulting estimated CATEs are therefore meaningless. The problem becomes worse as more irrelevant covariates are included in the distance metric.

Irrelevant covariates should be irrelevant to the outcome of matching. However, if the distance metric is chosen without consideration of the importance of covariates, in the most extreme case, irrelevant covariates could be chosen adversarially, so that the matching method could produce any given unreasonable match assignment. In Proposition \ref{theoremadversarial} below, we show how any unreasonable match assignment can appear to be reasonable, according to a fixed (not-learned) Hamming distance, if we are permitted to append a set of irrelevant variables to our dataset.

A \textit{match assignment}  (\ma) 
assigns an integer to each unit, and units with the same integer value are in the same \emph{matched group}. A \emph{reasonable match assignment} (\mar) is one in which each unit in a matched group is at most distance $D_{\max}$ from any other unit assigned to the same match group, using, for example, Euclidean distance or Hamming distance. For any two units, with vector of covariates $x_1$ and $x_2$, in a matched group created by \ma, the matching assignment is reasonable if distance$(x_1,x_2)\le D_{\max}$. Similarly, an \emph{unreasonable match assignment} (\maunm) is one in which there exists at least one treatment-control pair, with covariates respectively $x_1$ and $x_2$, in the same matched group where  $distance(x_1,x_2) > D_{max}$.

\begin{proposition}\label{theoremadversarial}
 Consider any matching method ``{\emph{Match}}'' that creates reasonable match assignments for dataset $\{X,T\}$. \emph{Match} is constrained so that if treatment unit $x_1$ and control unit $x_2$ obey Hamming $distance(x_1,x_2) \leq D_{max}$, and there are no other treatment or control units that are able to be matched with $x_1$ and $x_2$, then $x_1$ and $x_2$ will be assigned to the same matched group. 
 Consider any unreasonable match assignment \maunm\ for dataset $\{X,T\}$.  If $distance(\cdot,\cdot)$ is \emph{normalized by the number of covariates}, then it is possible to append a set of covariates $A$ to $X$ that: (i) can be chosen without knowledge of $X$, (ii) can be chosen without knowledge of $T$, (iii) matching on the appended dataset, \emph{Match}$(\{[X,A],T\})$, will make the match assignment \maunm\ reasonable (which is undesirable).
 \end{proposition}

 \begin{proof} 
 Consider unreasonable match $x_a$ and $x_b$ in \maunm. Then, create appended feature set $A$ to have more features than $X$, such that units $x_a$ and $x_b$ would have identical rows of $A$. These identical rows can be chosen without knowledge of the values within $x_a$ and $x_b$. This forces $x_a$ and $x_b$ to be matched. This would be repeated for every matched group in \maunm. This forces all matches in \maunm\ to occur.
 \end{proof}
 
 

This completely adversarial setting 
is unlikely to occur in reality, however, situations where matches are disrupted by irrelevant covariates are realistic. The more irrelevant covariates are included, the more the matched groups tend to disintegrate in quality for most matching methods. A reasonable sanity check is that irrelevant covariates should be able to be eliminated automatically by the matching method. 

As we discuss below, FLAME does not use a pre-determined distance for matching. It learns a distance for matching from a hold-out training set. In particular, it approximately solves the Full-AME problem formalized in the next subsection.

\subsection{Full Almost-Matching-Exactly (Full-AME) Problem} 
While matching using irrelevant covariates is problematic (as we have shown), matching on too few relevant covariates is also problematic.
\textit{We would like to ensure that each unit is matched using at least a set of covariates that is sufficient to predict outcomes well}. Conversely, if a unit is matched using a set of covariates that do not predict outcomes sufficiently well, we would not trust the results from its matched group. Let us formalize the problem of matching each unit on a set of covariates that together predict outcomes well.

 We will use $\btheta\in\{0,1\}^d$ to denote the variable selection indicator vector for a subset of covariates to match on. Throughout the following discussion, we consider a unit to be a triplet of (covariate value $x$, observed outcome $y$, treatment indicator $t$), unless otherwise stated. Given dataset $\mathcal{S}$, define the \textit{matched group} for unit $i$ with respect to covariates selected by $\btheta$ as the units in $\mathcal{S}$ that match $i$ exactly on the covariates $\btheta$:
$$
\MG_i (\btheta, \mathcal{S} ) = \{i^\prime \in \mathcal{S} : \x_{i^\prime} \circ \btheta = \x_i \circ\btheta \}.
$$
Here we use $\mathbf{x}_i$ and $\mathbf{x}_{i^\prime}$ to denote the covariate values for unit $i$ and $i^\prime$ respectively. Taking the union over all units, we also define the \textit{matched units} for a selection indicator $\btheta$ as 
{the collection of units that are matched on covariates defined by selection indicator $\btheta$:}
\begin{align}
    \MG (\btheta, \mathcal{S}) = \{i \in \mathcal{S} : \exists i \neq i^\prime \quad s.t. \quad   \x_{i^\prime} \circ \btheta = \x_i \circ\btheta \}. \label{eq:match-group} 
\end{align}
The value of a set of covariates $\btheta$ is determined by how well these covariates can be used together to predict outcomes. 
Specifically, the prediction error $\PE_{ \mathcal{F}_k } (\btheta)$ is defined with respect to a class of functions $\mathcal{F}_k := \{ f: \{0,1\}^k \rightarrow \mathbb{R} \}$ ($1 \le k \le d$) as: 
\begin{align}
\PE_{\mathcal{F}_{ \left\|\btheta \right\|_0} } (\btheta) = \min_{ f^{(1)} \in\mathcal{F}_{ \left\|\btheta \right\|_0}}  \mathbb{E} [ (f^{(1)} (\x \circ \btheta)-y)^2 | t = 1]  + \min_{f^{(0)} \in\mathcal{F}_{ \left\|\btheta \right\|_0}} \mathbb{E} [ (f^{(0)} (\x \circ \btheta)-y)^2 | t = 0]  ,
\label{eq:pe}
\end{align}
where the expectation is taken over $\x$ and $y$, when $\left\| \cdot \right\|_0$ is the count of nonzero elements of the vector. That is, $\PE_{\mathcal{F}_{ \left\|\btheta \right\|_0}} $ is the smallest prediction error we can get on both treatment and control populations using the features specified by $\btheta$.
Consider a separate training dataset $\mathcal{S}^{tr}$.
Let $\mathcal{S}_0^{tr}$ be the subset (of $\mathcal{S}^{tr}$) of control units $(X^{tr},Y^{tr})$ with $T^{tr} = 0$, and let $\mathcal{S}_1^{tr}$ be the subset (of $\mathcal{S}^{tr}$) of treated units $(X^{tr},Y^{tr})$ with $T^{tr} = 1$. The empirical counterpart of $\PE_{\mathcal{F}_{ \left\|\btheta \right\|_0}} $ is defined as: 
\begin{align}
\hat{\PE}_{\mathcal{F}_{ \left\|\btheta \right\|_0}} (\btheta, \mathcal{S}^{tr}) &= 
\min_{f^{(1)} \in \mathcal{F}_{ \left\|\btheta \right\|_0}} \frac{1}{ | \mathcal{S}_1^{tr} |} 
\sum_{ (\mathbf{x}_i, y_i) \in \mathcal{S}_1^{tr} } (f^{(1)}(\x_i^{} \circ \btheta )-y_i)^2\nonumber\\ 
&+ \min_{ f^{(0)} \in \mathcal{F}_{ \left\|\btheta \right\|_0}} \frac{1}{ | \mathcal{S}_0^{tr} |} 
\sum_{ (\mathbf{x}_i, y_i) \in \mathcal{S}_0^{tr} } (f^{(0)}(\x_i^{} \circ \btheta )-y_i)^2. 
\label{eq:pe-emp}
\end{align} 
Given a matching dataset $\mathcal{S}^{ma}$ and a training dataset  $\mathcal{S}^{tr}$, the best selection indicator we could achieve for a nontrivial matched group that contains treatment unit $i$ would be: 
\begin{eqnarray*}
\btheta_{i, \mathcal{S}^{ma}}^{*} \in \mathrm{arg}
\min_{\btheta} \hat{\PE}_{\mathcal{F}_{ \left\|\btheta \right\|_0} } (\btheta, \mathcal{S}^{tr})  \textrm{ s.t. }\exists \ell\in \MG_i ({\btheta}, \mathcal{S}^{ma}) \textrm{ s.t. } t_{\ell}=0  ,
\end{eqnarray*}
where $ t_{\ell} $ is the treatment indicator of unit $\ell$. This constraint says that the matched group contains at least one control unit. The covariates selected by $\btheta_{i, \mathcal{S}^{ma}}^{*}$ are those that predict the outcome best, provided that at least one control unit has the same exact covariate values as $i$ on the covariates selected by  $\btheta_{i, \mathcal{S}^{ma}}^{*}$. 

The \textit{main matched group} for $i$ is then defined as $\MG_i ( \btheta^*_{i, \mathcal{S}^{ma}}, \mathcal{S}^{ma})$.
\textit{The goal of the Full-AME problem is to calculate the main matched group} $\MG_i (\btheta^*_{i, \mathcal{S}^{ma}},  \mathcal{S}^{ma})$ \textit{for as many units $i$ as possible.} Once the problem is solved, the main matched groups can be used to estimate treatment effects, by considering the difference in outcomes between treatment and control units in each group, and possibly smoothing the estimates from the matched groups if desired, to prevent overfitting of treatment effect estimates. 

We now present a worst-case bound on the bias induced by matching units in an almost-exact framework, and using the created matched groups to estimate CATEs. 
We will see subsequently that the FLAME procedure directly targets minimization of this bias.

\subsection{Bias Bound in the Full-AME Problem}

If we do not match on all relevant covariates, a  bias is induced on the treatment effect estimates. As shown before, solving the Full-AME problem ensures that this bias is as small as possible, and zero if the covariates excluded are all irrelevant. Here we present a simple worst-case bound on the in-sample estimation bias when a CATE is estimated with units matched according to a chosen subset of covariates (defined by $\btheta$). This bound is worst-case in that it holds for any subset of relevant covariates. 
This implies that the bias resulting from Full-AME will be much smaller than the bound given here in most cases.

\begin{proposition} \label{thm:bound}
    Let $g^{(1)}(\x)$ and $g^{(0)}(\x)$ represent nonrandom potential outcomes at arbitrarily chosen covariate value $\x \in \{0, 1\}^p$, so that $ g^{(1)}(\x_i) := y_i^{(1)}$ and $ g^{(0)}(\x_i) := y_i^{(0)}$ are the potential outcomes in a sample of $n$ units. Fix a value of the covariates $\x \in \{0,1\}^d$, and a value of $\btheta \in \{0, 1\}^d$. Let $\MGS = \{i \in \mathcal{S}^{ma}\,:\, \x_i \circ \btheta = \x \circ \btheta\}$ be the set of units in the sample that have value equal to $\x$ on the covariates selected by $\btheta$. Define additionally: $\nxt = \sum_{i \in \MGS}T_i$ and $\nxc = \sum_{i \in \MGS}(1-T_i)$. Let $\tau(\x) = g^{(1)}(\x) - g^{(0)} (\x)$ be the CATE estimand of interest. For a weighted Hamming distance with positive weight vector $\bw$ of length $p$, and $0 < \| \bw \|_2 < \infty$, define $M = \max_{\substack{\x, \x' \in \{0, 1\}^p \\t \in \{0, 1\}}}\frac{|g^{(t)}(\x') - g^{(t)}(\x)|}{\bw^T\vneq{\x'}{\x}}$ and assume $M < \infty$. We have, for any $\btheta \in \{ 0, 1\}^p$: 
    \begin{align}
         &\left|\frac{1}{\nxt}\sum_{i \in \MGS}Y_iT_i - \frac{1}{\nxc}\sum_{i \in \MGS}Y_i(1-T_i) - \tau(\x)\right|\nonumber\\ &\hspace*{320pt} \le 2M\bw^T(\bone - \btheta).
    \end{align}
\end{proposition}

As asserted by Proposition \ref{thm:bound}, we should select $\btheta$ to minimize $\bw ^T (\bone - \btheta)$ in order to minimize FLAME's bias. In real problems, we should think about $\bw $ as a non-uniform vector that has some small entries. FLAME would tend to remove those entries so that the bias is minimized for the remaining covariates that are used for matching. A proof of Proposition \ref{thm:bound} can be found in Appendix \ref{app:proof-bias-bound}.
This bound  provides guidance in designing the FLAME procedure as it suggests that the amount of bias, in estimating treatment effects with almost-exact matching, depends heavily on $\btheta$: FLAME will try to match on sets of covariates that minimize such bias. 




\subsection{FLAME's Backward Procedure to Full-AME} 
While solving the Full-AME problem is ideal, computational challenges prevent its usage on large datasets. We thus resort to an approximated solution. In determining which covariates are important, FLAME leverages the classic Efroymson's backward stepwise regression \citep{efroymson1960multiple, mccornack1970comparison}, which is widely adopted by practitioners in regular regression tasks. This procedure starts with all covariates in the model, and greedily removes one covariate at a time, according to the \texttt{PE} and \texttt{BF} criteria, defined below. By using this backward procedure, FLAME is able to strike a good balance between using more units and using more covariates.  As evidenced in Section \ref{sec:experiments}, this approximate solution outperforms existing methods on tasks of interest.

Specifically, FLAME performs the following procedure: for iteration $i = 1, \cdots, d$, FLAME selects a set of covariates indicated by $\btheta^i$ according to a certain criterion, and stops when a stopping condition is met. We restrict FLAME to perform a greedy selection, namely we require $\btheta^i \succeq \btheta^{i+1}$ for all $i$. Here $ \succeq$ is the element-wise comparison. During this procedure, FLAME produces a sequence of selection indicators $\{ \btheta^0, \btheta^1, \cdots, \btheta^d \}$ ($\btheta^i \in \{0,1 \}^d$), where $\btheta^i$ denotes the selection indicator on iteration $i$. 




\subsubsection*{The \texttt{Prediction Error} (\texttt{PE}) and \texttt{Balancing Factor} (\texttt{BF}) Criteria} 
As stated in Proposition \ref{thm:bound}, the bias of almost-matching-exactly estimates depends on which covariates are selected for matching: matching on irrelevant covariates at the expense of relevant ones will increase bias. FLAME uses the \texttt{PE} Criterion (Eq. \ref{eq:pe}, \ref{eq:pe-emp}) to choose which covariates should be considered irrelevant and ignored when matching. 
FLAME further encourages larger and more balanced matched groups 
by using the \texttt{Balancing Factor} (\texttt{BF}). The \texttt{BF} is defined with respect to the selection indicators $\btheta$, and can be computed as follows: 
\begin{align}
    \BF (\mathcal{MG} (\btheta, \mathcal{S}^{ma} ) ) = \frac{ \text{\# control in } \mathcal{MG} (\btheta, \mathcal{S}^{ma} ) }{\text{\# available control}} + \frac{ \text{\# treated in } \mathcal{MG} (\btheta, \mathcal{S}^{ma} ) }{\text{\# available treated}}. 
\end{align} 
Maximizing \texttt{BF} encourages a large fraction of both treatment and control units to be used for the matched groups. This implies that more units would be matched in earlier iterations; as we know, it is better to match units in earlier iterations because the matches made in earlier iterations depend on more variables and are thus higher quality.

\subsubsection*{FLAME's Procedure} 
FLAME's backward selection criterion is a combination of  \texttt{PE} and \texttt{BF}. More specifically, at each iteration  $i = 0,1,2, \cdots, p$, given a training dataset $\mathcal{S}^{tr}$ and matching dataset $\mathcal{S}^{ma}$, FLAME finds a selection indicator such that 
\begin{align}
    \btheta^i \in & \arg \max_{ \btheta } \left[ - \hat{ \mathrm{\texttt{PE}}}_{ \mathcal{F}_{ \left\|\btheta \right\|_0}} (\btheta, \mathcal{S}^{tr} ) + C \cdot \mathrm{\texttt{BF}} (\mathcal{MG} (\btheta, \mathcal{S}^{ma} )) \right]  \nonumber
    \\&\;\; s.t. \;\;  \btheta^i \in \{0,1 \}^d , \|\btheta^i\|_0 = d - i, \btheta^i \preceq \btheta^{i-1}.  \label{eq:criterion}
\end{align} 
Here $\preceq$  is for element-wise comparison, and $C$ is a hyperparameter that trades off \texttt{PE} and \texttt{BF}. 

Similarly to a general backward method, FLAME retains a threshold that prevents the algorithm from eliminating too many covariates. That is, FLAME retains a stopping criterion that ensures \emph{every matched group is matched exactly on a set of covariates that, together, can predict outcomes well}  (sacrificing no more than $\epsilon$ training accuracy on the hold-out training set). 
Let us first define feasible covariate sets for treatment unit $u$ as those covariate sets that predict well on the hold-out training set (and lead to a valid matched group for $u$): 
\begin{align*} 
&\bar{\btheta}^{\textrm{feasible}} = \left\{\btheta \in \{ 0,1 \}^d: \PE_{\mathcal{F}_{ \left\|\btheta \right\|_0}} (\btheta,\mathcal{S}^{tr}) \leq \PE_{\mathcal{F}_d} (\textbf{1}_{d \times 1},\mathcal{S}^{tr}) +\epsilon 
\right\}. 
\end{align*}
To incorporate this feasible set constraint, we simply add the requirement $\btheta^i \in \bar{\btheta}^
{\text{feasible}}$ to (\ref{eq:criterion}). 

In Algorithm \ref{alg:FLAME}, we summarize the FLAME procedure. The core task of the algorithm is to determine a sequence of variable selection indicators $\{ \btheta^0, \btheta^1, \cdots, \btheta^d \}$. Matches can be easily determined given the sequence of variable selection indicators. In order to formally state the matching procedure, we use $\mathcal{S}^{ma} = (X, Y, T)$ to denote the triplet of (covariates, outcomes, treatment indicators) of the units for matching, $\mathcal{S}^{tr} = (X^{tr}, Y^{tr}, T^{tr})$ to denote the same for units in the training set.






\begin{algorithm}[t]
\label{alg:flame}
\begin{algorithmic}[1] 
    \Statex \textbf{Inputs} Input data $\mathcal{S}^{ma}=(X,Y,T)$ for matching; 
    training set $\mathcal{S}^{tr}=(X^{tr}, Y^{tr}, T^{tr})$; model classes $\mathcal{F}_1, \mathcal{F}_2, \cdots, \mathcal{F}_d $; stopping threshold $\epsilon$; tradeoff parameter $C$. 
    \Statex \textbf{Outputs} A sequence of selection indicators $ \btheta^0, \cdots, \btheta^d $, and a set of matched groups $\{\MG ( \btheta^\iter, \mathcal{S}^\iter ) \}_{\iter \geq 1}$.
    \Comment{$\mathcal{S}^\iter$ \textit{is defined in the algorithm.} }
    \State Initialize $\mathcal{S}^0 = \mathcal{S}^{ma} = (X,Y,T), \btheta^0 = \mathbf{1}_{d \times 1}, \iter=1, run = True$. 
    \Statex \Comment{ $\iter$ \textit{is the index for iterations.}} 
    
     
    \State Compute exact matched groups $\MG (\btheta^0, \mathcal{S}^0 )$ as defined in (\ref{eq:match-group}). 
    \Statex \Comment{\textit{The detailed implementation is in Section \ref{sec:algorithm-implementation}.}}
     
    \While{$run=True$ } \label{step:while}
  
        \State Compute $\btheta^\iter$ using (\ref{eq:criterion}) on training set $\mathcal{S}^{tr}$, using $\mathcal{F}_{d-l}$ and tradeoff parameter $C$. 
        \Statex \Comment{\textit{Determine which covariates to match on for this iteration.}}
        \State Compute matched groups $\MG (\btheta^{\iter - 1}, \mathcal{S}^{\iter - 1})$ as defined in (\ref{eq:match-group}). 
        \Statex \Comment{\textit{The detailed implementation is in Section \ref{sec:algorithm-implementation}.}}
        \State $\mathcal{S}^{\iter} = \mathcal{S}^{\iter - 1} \setminus \MG (\btheta^{\iter - 1}, \mathcal{S}^{\iter - 1})$. \Comment{\textit{These  matched units are done.}}
            \If{ $\hat{ \mathrm{\texttt{PE}}}_{\mathcal{F}_{d-l}}(\btheta^l, \mathcal{S}^{tr} ) > \hat{ \mathrm{\texttt{PE}}}_{\mathcal{F}_d}(\mathbf{1}_{d \times 1}, \mathcal{S}^{tr} ) + \epsilon$ \textrm{ OR } $\mathcal{S}^{\iter} = \emptyset$ }
                \State {$run=False$ }
                \Comment{\textit{Prediction error is too high to continue matching.}}
            \EndIf
      
        
  
    \State $\iter = \iter+1$
    \EndWhile
    
    \State \textbf{Output} $\{ \btheta^l, \MG (\btheta^\iter, \mathcal{S}^\iter) \}_{\iter \geq 1}$.  
    \end{algorithmic} 
    \caption{: FLAME Algorithm}
    \label{alg:FLAME}
\end{algorithm}
The algorithm starts by initializing the selection indicator to include all covariates (Line 1). Exact matches are made when possible at Line 2. At each iteration of the while loop, FLAME computes the selection indicator for that iteration by minimizing the objective defined in \eqref{eq:criterion} (Line 4) and matches units exactly on the covariates selected by the newly found selector (Line 5). Matched units are then excluded from the leftover data (Line 6), and the procedure is repeated until either prediction error increases too much from removing an additional covariate, or all the data are matched (Line 7). 

FLAME's implementation performs matching without replacement, but can be adjusted to perform matching with replacement.

An estimate of the ATE is straightforward to compute once the treatment effects in each group (conditional average treatment effects -- CATEs) are computed as differences in outcome means between treatment and control units. For better CATE estimates, smoothing (e.g., regression) can be performed on the raw matched CATE estimates.

\section{Implementing Matched Groups in FLAME} \label{sec:algorithm-implementation}
The workhorse behind FLAME's database implementation is its matching subroutine of finding $\MG$'s. We implement this procedure using the following two methods. 



\subsection{Implementation 
using Database (SQL) Queries}\label{sec:algo-db}


Exact matching is highly related to the {\tt GROUP} \texttt{BY} operator used in database (SQL) queries,  which computes aggregate functions (sum, count, etc) on groups of rows in a two-dimensional table having the same values of a subset of columns specified in the query. SQL queries can be run on any standard commercial or open-source relational database management system (\eg, Microsoft SQL Server, Oracle, IBM DB2, Postgres, etc.). 
These database systems are highly optimized and robust for SQL queries, can be easily integrated with other languages (we used python and SQL), and scale to datasets with a large number of rows ($n$) or columns ($p$) that may not fit in the available main memory. 
In addition, SQL queries \emph{declaratively} specify complex operations (we only specify `what' we want to achieve, like matched groups on the same values of variables, and not `how' to achieve them, \ie, no algorithm has to be specified), and are therefore succinct. In our work, SQL enables us to execute a matching step  
in a single query as we discuss below. 
In this implementation, we keep  track of matched units globally by keeping an extra column called {\tt is\_matched} in the input database $\mathcal{S}^{ma}$ (containing the data triplet $\mathcal{S}^{ma} = (X,Y,T)$ described in Section \ref{sec:flame-framework}).
 For every unit, the value of {\tt is\_matched = $\ell$} if the unit was matched in a valid main matched group with at least one treated and one control unit in iteration $\ell$ of Algorithm~\ref{alg:FLAME}, and {\tt is\_matched = 0} if the unit is still unmatched.  
For notational simplicity, let $A_1, \cdots, A_k$ be the names of the covariates selected by $\btheta$, and $S=\mathcal{S}^{ma}$. The SQL query for the matching procedure without replacement on dataset $\texttt{S}$ is given below.
{\small
\begin{alltt}
\quad WITH tempgroups AS 
\quad \quad (SELECT \(A\sb{1}, A\sb{2}, \cdots, A\sb{k}\) 
\hspace*{95pt} \textit{--(matched groups will be identified by their covariate values)}
  \quad \quad FROM S
  \quad \quad WHERE  \(\ismatched\ = 0\) 
\hspace*{95pt} \textit{--(use data that are not yet matched)}
  \quad \quad GROUP BY \(A\sb{1}, A\sb{2}, \cdots, A\sb{k}\) 
\hspace*{95pt} \textit{--(create matched groups with identical values of covariates)}
  \quad \quad HAVING SUM(T) > 0 AND SUM(T) < COUNT(*) 
\hspace*{95pt} \textit{--(groups have at least one treated and one control unit)}
\quad \quad )
\quad UPDATE S
\quad SET \(\ismatched\ = \ell\)
\quad WHERE \(\ismatched\ = 0\) AND
\quad \quad EXISTS
\quad \quad \quad (SELECT \(\texttt{Q}.A\sb{1}, \texttt{Q}.A\sb{2}, \cdots, \texttt{Q}.A\sb{k}\)
\quad \quad \quad  FROM tempgroups AS Q \hspace*{10pt} 
\hspace*{95pt} \textit{--(set of covariate values for valid groups)}
\quad \quad \quad  WHERE  \(\texttt{Q}.A\sb{1} = \texttt{S}.A\sb{1}\) AND \(\texttt{Q}.A\sb{2} = \texttt{S}.A\sb{2}\) AND  \(\cdots\) AND \(\texttt{Q}.A\sb{k} = \texttt{S}.A\sb{k}\)) 
\end{alltt}
}

The \emph{WITH clause} computes a temporary relation \emph{tempgroups} that computes the combination of values of the covariates forming `valid groups' (\ie, groups with at least one treatment and at least one control unit) on unmatched units. The \emph{HAVING clause} of the SQL query discards groups that do not satisfy this property -- since treatment $T$ takes binary values ($0$ or $1$), for any valid group, the sum of $T$ values will be strictly greater than $0$ and strictly less than the total number of units in the group.  Then we update the  population table \texttt{S}, where the values of the covariates of the existing units match with those of a valid group in \emph{tempgroups}.  Several optimizations of this basic query are possible and are used in our implementation. Setting the {\tt is\_matched} value to level $\ell$ (instead of a constant value like 1) helps us compute the CATE for each matched group efficiently. (Note that for each value of $\ell$, there might be more than one matched group, with different values of covariates being used for matching.)

\subsection{Implementation 
using Bit Vectors}\label{sec:algo-bit}

In an alternative bit-vector implementation, for binary covariates, we encode the combination of the to-match-on covariates of unit $i$ as a binary number $b_i$; we also encode the to-match-on covariates, appended with the treatment indicator as the least significant digit, to form a binary number $b_i^+$. A unit $i$ with ordered to-match-on variable values $(a_{i, d}, a_{i, d-1}, \cdots, a_{i, 1})$ with $a_{i,k} \in \{0, 1 \}$ and treatment indicator $t_i \in \{0,1 \}$ is associated with numbers $b_i = \sum_{k=1}^d a_{i,k} 2^{k-1}$ and $b_i^+  = \sum_{k=1}^d a_{i,k} 2^{k} + t_i$. Two units $i$ and $j$ have the same covariate values if and only if $b_i = b_j$. For each unit $i$, we count how many times $b_i$ and $b_i^+$ appear (in the whole dataset), and denote the counts as $c_i$ and $c_i^+$ respectively. 
A unit $i$ is matched if and only if $c_i \neq c_i^+ $, since the two counts differ \textit{if and only if} the same $b_i$ appears both as a treated instance and a control instance. This property is summarized in Proposition \ref{prop:bit-vec}. 
For non-binary categorical data, if the $k$-th covariate is $h_{(k)}$-ary, we first rearrange the $d$ covariates such that $h_{(k)} \le h_{(k+1)} $ for all $1\le k \le d-1$. 
Thus each unit $(a_{i,d}, a_{i,d-1}, \cdots, a_{i,1})$ uniquely represents the number $\sum_{k=1}^{d} a_{i,k} h_{(k)}^{k-1}$. From here we can apply the above method for the binary case.

\begin{proposition}
\label{prop:bit-vec}
A unit $u$ is matched if and only if $c_u \neq c_u^+$, since the two counts $b_u$ and $b_u^+$ differ if and only if the same combination of covariate values appear both as a treated unit and a control unit.
\end{proposition}
An example of this procedure is illustrated in Table \ref{tab:bit-vec-example}. We assume in this population the first variable is binary and the second variable is ternary. In this example, the number $b_1$ for the first unit is $0 \times 2^0 + 2 \times 3^1 = 6$; the number $b_1^+$ including its treatment indicator is $0  + 0 \times 2^1 + 2 \times 3^2 = 18$. Similarly, we can compute all the numbers $b_u, b_u^+, c_u, c_u^+$, and the matching results are listed in the last column in Table \ref{tab:bit-vec-example}. 
 
\begin{table}[t]
\centering
\footnotesize
\begin{tabular}{ | c | c | c | c | c | c | c | c | }
  \hline
  first variable & second variable & T & $b_i$ & $b_i^+$ & $c_i$ & $c_i^+$ & is matched? \\ \hline \hline
  0 & 2 & 0 & 6 & 18 & 1 & 1 & No  \\ \hline
  1 & 1 & 0 & 4 & 11 & 2 & 1 & Yes  \\ \hline
  1 & 0 & 1 & 1 & 3 & 1 & 1 & No  \\ \hline
  1 & 1 & 1 & 4 & 12 & 2 & 1 & Yes \\ \hline
\end{tabular}
\caption{Example population table illustrating the \textit{bit-vector} implementation. Here the second unit and the fourth unit are matched to each other while the first and third units are left unmatched. \label{tab:bit-vec-example}}
\end{table}


\subsection{Comparison}
The bit vector implementation typically outperforms the SQL implementation when the data fits in memory, but for large data and more covariates, SQL performs better. A detailed comparison is deferred to Section~\ref{sec:experiments}. Another limitation of the bit vector implementation is that the magnitude of the numeric representation grows exponentially and can cause overflow problems. This is another reason to use FLAME-db when the number of covariates or total number of categories is large. 

\section{Bias Calculations}\label{sec:theory}
In this section we provide two types of bias calculation. One (Theorems \ref{thm:twocov}, \ref{thm:threecov}, \ref{thm:noalpha}) provides a bias calculation with oracle covariate importance information, while the other (Proposition \ref{prop:empirical-bias}) serves as the empirical counterpart of Proposition \ref{thm:bound}. 
\subsection{Exact Bias Computations for Oracle FLAME}
FLAME trades off statistical bias for computational speed. Here we provide insight into the bias that an oracle version of FLAME without replacement (defined below) induces when estimating heterogeneous causal effects as well as its unbiased performance when estimating an average causal effect. To evaluate the theoretical behavior of the algorithm we consider the outcome model 
$$
y_i = \alpha_0 + \sum_{j=1}^d \alpha_jx_{ij} + \beta_0 T_i + T_i\sum_{j=1}^d \beta_jx_{ij}
$$ 
as a data generation process that corresponds to a treatment effect $\beta_j$ being associated with every covariate $x_{ij}$ for individual $i$. Here $y_i$ is the observed outcome and $T_i$ is the observed treatment indicator. We are interested in the bias of FLAME for this simple non-noisy outcome model.
We define the Oracle FLAME algorithm as a simplified version of Algorithm \ref{alg:FLAME} that knows the correct order of importance for the covariates. Without loss of generality, let that order be $d,d-1,\dots,1$. Given this ordering, we can directly compute the bias of the FLAME estimates for various combinations of covariates. 

To compute the overall bias of Oracle FLAME, in theory we would enumerate all possible covariate allocations and run Oracle FLAME on each of those. 
For example, in the two-covariate setting with possible attribute values 0 and 1, there are $2^2\times 2^2=16$ possible covariate allocations for treatment and for control units leading to $16\times 16=256$ total possible allocations. Since we are interested only in the bias induced by the algorithm itself, we consider only treatment-control allocations where our procedure yields an estimate for each covariate combination. In cases where we do not have an estimate of treatment effect for each covariate combination, we cannot calculate the bias of the algorithm for any distribution that has support over the full covariate space; Oracle FLAME's bias estimates would not be defined on part of the covariate space in these cases.
Note that this is different than the standard overlap assumption made in most theory for causal inference as this calculates the bias when there is overlap at {\it some point} during the FLAME procedure.\footnote{We note that under the standard overlap assumptions, as those used in Section~\ref{sec:empirical_bound}, all units would be matched in the first step.} 
For example, in the two covariate setting, the allocation of a total of one treated and one control unit with both having $x_1=x_2=0$ would not be considered in our bias calculation since the exact matching procedure would obtain
 an estimate for that covariate combination only and not for the other three. 
An allocation consists of a set of covariate values, and treatment indicators. (The number of units in each covariate bin does not matter, what matters is that there is at least one treatment and one control in the bin, so that we can compute the bin's treatment effect.) For instance, one of the valid allocations would have at least one treatment and one control unit in each bin (a bin is a combination of covariates). Another valid allocation would have treatment and control units in most (but not all) bins, but when the bins are collapsed according to Oracle FLAME, each bin still receives an estimate.
 We perform these computations for two and three binary covariates and use the results to provide intuition for when  we have arbitrarily many covariates. The main results are as follows.

\begin{theorem}[Two covariates]
(i) There are 59 valid allocations. 
(ii) Under a uniform distribution over valid allocations, the biases (taking expectation over the allocations) are given by:
\begin{align*}
bias &= (\text{expected TE under FLAME})-(\text{actual TE}) \\
&=
\begin{blockarray}{ccc}
& x_2=0 & x_2=1   \\
\begin{block}{c(cc)}
 x_1=0 & \frac{20\beta_1 + 41/2\beta_2}{59} & \frac{20 \beta_1 - 41/2\beta_2}{59} \\
  x_1=1 & \frac{-20\beta_1 + 41/2\beta_2}{59} & \frac{-20 \beta_1 - 41/2\beta_2}{59} \\
\end{block}
\end{blockarray}
.
\end{align*}
\label{thm:twocov}
\end{theorem}

\begin{theorem}[Three covariates]
(i) There are 38070 valid allocations.
(ii) Under a uniform distribution over valid allocations the biases (taking expectation over the allocations) are given by:
\begin{align*}
\vspace{-1cm}
\begin{blockarray}{cccc}
&& x_2 = 0   &\\
\begin{block}{cc(c)c}
\BAmultirow{5pt}{$x_1$} & 
 0 & 
\frac{(5976 + 34/105)\beta_1 + (7854 + 61/210)\beta_2 + 11658\beta_3}{38070} &\BAmultirow{20pt}{,\;${x_3=0}$} \\
  &1 & 
  \frac{-(5976 + 34/105)\beta_1 + (7854 + 61/210)\beta_2 + 11658\beta_3}{38070} &\\
\end{block}
\end{blockarray}
\end{align*}
\vspace{-0.9cm}
\begin{align*}
\begin{blockarray}{cccc}
&& x_2 = 1   &\\
\begin{block}{cc(c)c}
\BAmultirow{10pt}{$x_1$} & 
 0 & 
\frac{(12755+6/7)\beta_1 - (16513+4/21)\beta_2 + 19035\beta_3}{38070} &\BAmultirow{20pt}{,\;${x_3=0}$} \\
  &1 & 
  \frac{-(12755+6/7)\beta_1 - (16513+4/21)\beta_2 + 19035\beta_3}{38070} &\\
\end{block}
\end{blockarray}
\end{align*}
\vspace{-0.9cm}
\begin{align*}
\begin{blockarray}{cccc}
&& x_2 = 0   &\\
\begin{block}{cc(c)c}
\BAmultirow{10pt}{$x_1$} & 
 0 & 
\frac{(12755+6/7)\beta_1 +(16513+4/21)\beta_2 - 19035\beta_3}{38070} &\BAmultirow{20pt}{,\;${x_3=1}$} \\
  &1 & 
  \frac{-(12755+6/7)\beta_1 + (16513+4/21)\beta_2 - 19035\beta_3}{38070} &\\
\end{block}
\end{blockarray}
\end{align*}
\vspace{-0.9cm}
\begin{align*}
\begin{blockarray}{cccc}
&& x_2 = 1   &\\
\begin{block}{cc(c)c}
\BAmultirow{10pt}{$x_1$} & 
 0 & 
\frac{(5976+34/105)\beta_1 - (7854+61/210)\beta_2 - 11658\beta_3}{38070} &\BAmultirow{20pt}{,\;${x_3=1.}$} \\
  &1 & 
  \frac{-(5976+34/105)\beta_1 - (7854+61/210)\beta_2 - 11658\beta_3}{38070} &\\
\end{block}
\end{blockarray}
\end{align*}

\label{thm:threecov}
\end{theorem}
A proof is available in the appendix. 

As these theorem show, the FLAME estimates are biased only by fractions of the treatment effects associated with the covariates ($\beta_j$ for $j>0$) rather than any baseline information (the $\alpha$'s) or a universal treatment effect ($\beta_0$). This is an appealing quality as it suggests that \textit{rare covariates that have large effects on baseline outcomes are unlikely to have undue influence on the bias of the treatment effect estimates}. This result can be made more concrete for an arbitrary number of covariates:

\begin{theorem}
The bias of estimating heterogeneous causal effects due to Oracle FLAME is not a function of $\alpha_0,\dots,\alpha_d$.
\label{thm:noalpha}
\end{theorem}
A proof of this theorem is available in the appendix. 

Theorems \ref{thm:twocov}, \ref{thm:threecov}, and \ref{thm:noalpha} further mean that in cases where the causal effect is homogeneous (that is, $\beta_j=0$ for $j>0$), this effect can be estimated without any bias. 
In that sense, FLAME provides an advantage over matching methods that target average treatment effect (such as propensity score matching) since FLAME handles potential heterogeneity at no additional cost. 


\subsection{Empirical Bound for FLAME Bias}\label{sec:empirical_bound} 

In this section, we discuss an empirical counterpart to Proposition~\ref{thm:bound} that provides a bias bound in terms of the \texttt{PE} (and \texttt{BF}) values during the FLAME procedure (Algorithm \ref{alg:FLAME}). 

\begin{proposition} 
\label{prop:empirical-bias}
Let $p_0 (X)$ and $p_1 (X)$ be the covariate density functions for the control group and the treated group. Assume\footnote{Note that this is equivalent to the standard assumption that the density of the covariates is bounded away from 0 \citep[e.g.,][]{imbens2015}.} that for any $ k \in \{1,2, \cdots,d\} $, there exists a constant $\lambda_{k } > 0$ such that the marginals $ p_0 \left( X \circ \btheta = \mathbf{x} \circ \btheta \right) \ge \lambda_k $ and  $ p_1 \left( X \circ \btheta = \mathbf{x} \circ \btheta \right) \ge \lambda_{ k } $ for any $\mathbf{x}$ and $\btheta$ with $ \left\| \btheta \right\|_0 = k$.
For a dataset $\mathcal{S}^{ma} = (X,Y,T)$, let $\mathcal{S}_0^{ma}$ be the subset of control units: $(X,Y)$ with $T = 0$, and let $\mathcal{S}_1^{ma}$ be the subset of treated units $(X,Y)$ with $T = 1$. For a dataset $\mathcal{S}^{ma}$, we overload notation by using $\x\in\mathcal{S}^{ma}$ and $(\x,y)\in\mathcal{S}^{ma}$ to refer to a covariate record in $\mathcal{S}^{ma} $ and a (covariate,outcome) pair record in $\mathcal{S}^{ma}$. 
Also, for a dataset $\mathcal{S}^{ma}$ (or $\mathcal{S}_0^{ma}$, $\mathcal{S}_1^{ma}$), let $$
n (\mathbf{x}, \btheta, \mathcal{S}^{ma} ) := \sum_{ \mathbf{x}_i \in \mathcal{S}^{ma} } \mathbbm{1} [ \mathbf{x}_i \circ \btheta = \mathbf{x} \circ \btheta ] 
$$
be the number of points that agree with $\mathbf{x}$ on the covariates defined by $\btheta$.
For any $f \in \mathcal{F}_{\left\| \btheta \right\|_0}$, $t \in \{0,1\}$, and datasets $\mathcal{S}_0^{ma}$, $\mathcal{S}_1^{ma}$, write 
$$
\hat{\epsilon}^{\max} ( \mathcal{S}_t^{ma}, f, \btheta ) := \max_{(\mathbf{x},y) \in \mathcal{S}_t^{ma} } \left| f ( \mathbf{x} \circ \btheta) - y \right|.
$$
Let $\hat{g} (\mathbf{x}, \btheta, \mathcal{S}_1^{ma})$ 
(resp. $\hat{g} (\mathbf{x}, \btheta, \mathcal{S}_0^{ma})$) be the plain estimator of the treated (resp. control) outcome for covariate values $\mathbf{x} \circ \btheta$: 
\begin{align*}
    \hat{g} (\mathbf{x}, \btheta, \mathcal{S}_1^{ma}) &:= \frac{1}{ | n (x, \btheta, \mathcal{S}_1^{ma} ) | } \sum_{ (\mathbf{x}_i, y_i) \in \mathcal{S}_1^{ma} } y_i \mathbbm{1} [ \mathbf{x}_i \circ \btheta = \mathbf{x} \circ \btheta ],  \text{and} \\
    \hat{g} (\mathbf{x}, \btheta, \mathcal{S}_0^{ma}) &:= \frac{1}{ | n (x, \btheta, \mathcal{S}_0^{ma} ) | } \sum_{ (\mathbf{x}_i, y_i) \in \mathcal{S}_0^{ma} } y_i \mathbbm{1} [ \mathbf{x}_i \circ \btheta = \mathbf{x} \circ \btheta ].
\end{align*}
Assume that the outcomes are noiseless and there are $r$ relevant covariates. Then for any fixed $\btheta$, $f^{(0)}, f^{(1)} \in \mathcal{F}_{\left\| \btheta \right\|_0}$, with probability at least 
\begin{align}
    1 -  \left( 1 - \lambda_{\left\| \btheta \right\|_0 + r} \right)^{  |  \mathcal{S}_0^{tr} | } - \left( 1 - \lambda_{\left\| \btheta \right\|_0 + r} \right)^{ |  \mathcal{S}_1^{tr}  | }  - \left( 1 - \lambda_{\left\| \btheta \right\|_0 + r} \right)^{ | \mathcal{MG}  ( \btheta, \mathcal{S}_0^{ma}) | } - \left( 1 - \lambda_{\left\| \btheta \right\|_0 + r} \right)^{  | \mathcal{MG} ( \btheta, \mathcal{S}_1^{ma}) | }, \label{eq:prob} 
\end{align} 
we have 
\begin{align} 
    \left| \hat{g} (\mathbf{x}, \btheta, \mathcal{S}_1^{ma}) - g^{(1)} (\mathbf{x}) \right| &\le 2 \hat{\epsilon}^{\max} ( \mathcal{S}_1^{tr}, f^{(1)}, \btheta ) \quad \text{and}\nonumber\\
    \left| \hat{g} (\mathbf{x}, \btheta, \mathcal{S}_0^{ma}) - g^{(0)} (\mathbf{x}) \right| &\le 2 \hat{\epsilon}^{\max} ( \mathcal{S}_0^{tr}, f^{(0)}, \btheta ),  \label{eq:max-error-bound}
\end{align} 
where $\mathcal{S}_0^{tr}$ (resp. $\mathcal{S}_1^{tr}$) is the collection of control (resp. treated) units in the holdout set, and $g$ is the true treatment outcome defined as in {Proposition \ref{thm:bound}}. 

\end{proposition}
Proof of this proposition is available in the appendix. 
During a FLAME run, part of FLAME's criterion is to minimize \texttt{PE} and therefore the corresponding $\hat{ \epsilon }^{\max}$ values are controlled. This is because $\hat{\epsilon}^{\max} ( \mathcal{S}^{ma}, f, \btheta ) := \left\| \hat{\mathbf{y}} - \mathbf{y} \right\|_\infty $, where $\mathbf{y}$ is the vector formed by labels from $\mathcal{S}^{ma}$ and $\hat{\mathbf{y}}$ is the vector formed by predictions of instances in $ \mathcal{S}^{ma} $ using $f$. Since $\hat{\PE}$ is proportional to the square of the $L_2$ norm of the same vector on the training set, minimizing $\hat{\PE}$ on the training set tends to control $\hat{\epsilon}^{\max}$ on the matching set. 
Also, part of FLAME's criterion maximizes the \texttt{BF} values, which implicitly maximizes $n (\mathbf{x}, \btheta, \mathcal{S}_0^{tr}),   n (\mathbf{x}, \btheta, \mathcal{S}_1^{tr} ),  n (\mathbf{x}, \btheta, \mathcal{S}_0^{ma})$, and $n (\mathbf{x}, \btheta, \mathcal{S}_1^{ma}) $. This in turn suggests the probability in (\ref{eq:prob}) is large. 

Proposition \ref{prop:empirical-bias} can be linked to the almost-exact-matching problem in the following way: when $r \ll p$, and 
$ \lambda_k = \Theta (\frac{1}{2^k}) $, we have $(1 - \lambda_{\left\| \btheta \right\|_0 + r} )^n$ being small for small $ \left\| \btheta \right\|_0 $ and large $n$. The implication of this is that, 
in a FLAME run with $r \ll p$, when many irrelevant covariates are eliminated, $\left\| \btheta \right\|_0 $ becomes small and $| \mathcal{S}_0^{tr} |, | \mathcal{S}_1^{tr} |, | \mathcal{MG} ( \btheta, \mathcal{S}_0^{tr}) |, | \mathcal{MG} ( \btheta, \mathcal{S}_1^{tr}) | $ become large. 
In this case the probability in (\ref{eq:prob}) is large. 
This serves as a type of justification for the FLAME procedure, since FLAME aims to be able to handle problems where $ r \ll p $. As a result of $ r \ll p $, we need only  concentrate on  matches (on $\btheta$) with $\left\| \btheta \right\|_0 \ll p$.

\section{Experiments}\label{sec:experiments}%
 In this section, we study the quality and scalability of FLAME on synthetic and real data. The real datasets we use are the US Census 1990 dataset 
 from the UCI Machine Learning Repository \citep{Lichman:2013} and the US 2010 Natality data \citep{natality2010}. The bit-vector and SQL implementations are referred to as {\bf FLAME-bit} and {\bf FLAME-db} respectively. FLAME-bit was implemented using Python 2.7.13 and FLAME-db using Python, SQL, and Microsoft SQL Server 2016. 
We compared FLAME with several other (matching and non-matching) methods including: (1) \textbf{one-to-one Propensity Score Nearest Neighbor Matching (1-PSNNM)} \citep{ross2015propensity}, (2) 1-PSNNM with oracle variable selection, (3) \textbf{Genetic Matching} (GenMatch) \citep{genmatch}, (4) \textbf{Causal Forest} \citep{wager2018estimation}, (5) \textbf{Mahalanobis Matching}, (6) \textbf{double linear regression}, (7) \textbf{BART} \citep{chipman2010bart} and (8) \textbf{CTMLE} \citep{van2006targeted}. 
 For 1-PSNNM with oracle variable selection, we reveal to the propensity score matcher the \textit{true} important covariates for outcomes and/or the important covariates for propensity score, which covers the ideal case for propensity scare matching with variable reweighting \citep{schneeweiss2009high}.
 In double linear regression and BART, we fit two regressions, one for the treatment group and one for the control group; the estimate of treatment effect is given by the difference of the predictions of the two regressors. 
(Unlike the matching methods, the double regression method assumes a model for the outcome making it sensitive to misspecification. 
Here we correctly specify the linear terms of the generative model in the synthetic data experiments.) We also ran tests with \textbf{Coarsened Exact Matching} (CEM) \citep{iacus2011causal} and \textbf{Cardinality Match} \citep{zubizarreta2012using, zubizarreta2014matching}. 
CEM is not well suited for categorical data 
and the R package does not automatically handle more than 15 covariates. 
Cardinality match does not scale to the setting of our problem, since it tries to solve a computationally demanding mixed integer programming problem. In all experiments with synthetic data, the training set is generated with exactly the same setting as the dataset used for matching, unless otherwise specified. The computation time results are in Table \ref{tab:compare-timing-census}. 
The experiments were conducted on a Windows 10 machine with  
Intel(R) Core(TM) i7-6700 CPU processor (4 cores, 3.40GHz, 8M) and 32GB RAM.

\subsection{Experiments with Synthetic Data}
Since the true treatment effect cannot be obtained from real data, we created simulated data with specific characteristics and known treatment effects.
Note that both {FLAME-db} and {FLAME-bit} always return the same treatment effects. Four of the simulated experiments use data generated from special cases of the following (treatment $T\in\{0,1\}$):
\begin{align}
y = \sum_{i=1}^{10} \alpha_i x_i + T \sum_{i=1}^{10} \beta_i x_i + T\cdot U \sum_{1\leq i<\gamma\leq 5} x_i x_{\gamma} + \epsilon, \label{syntheticdata}
\end{align} 
Here, $\alpha_i \sim N(10s, 1) $  with $s \sim \text{Uniform}\{-1, 1 \}$, $\beta_i \sim N(1.5, 0.15)$, $U$ is a constant and $\epsilon \sim \mathcal{N} ( 0 , 0.1 )$. This contains linear baseline effects and treatment effects, and a quadratic treatment effect term.

Below we present the main comparison of our approach to state-of-the-art causal methods. Section \ref{subs:irrelevant} provides evidence of significant improvements in terms of estimation of causal effects while Section \ref{subsec:exp-scalability} demonstrates the scalability of the method. Additional results on the effects of model misspecification on regression methods, scalability comparison of FLAME-bit and FLAME-db, the decay of performance for FLAME as less important variables are eliminated, and the effect of the tuning parameter $C$ are available in Appendix \ref{app:add-exp}.



\vspace{-0.2cm}
\subsubsection{Most matching methods cannot handle irrelevant variables}\label{subs:irrelevant}
\vspace{-0.2cm}
Most matching methods do not make a distinction between important and unimportant variables, and may try to match on all the variables, including irrelevant ones. Therefore, many matching methods may perform poorly in the presence of many irrelevant variables.


We consider 20,000 units (10,000 control and 10,000 treated) generated with (\ref{syntheticdata}) where $U=1$. We also introduce 20 irrelevant covariates for which $\alpha_i=\beta_i=0$, but the covariates are in the database. We generate $x_i \sim \text{Bernoulli}(0.5)$ for $1\le i \le 10$. For $10 < i \le 30$, $x_i \sim \text{Bernoulli}(0.1)$ in the control group and $x_i \sim \text{Bernoulli}(0.9)$ in the treatment group.


During the execution of FLAME, irrelevant variables are successively dropped before the important variables. We should stop FLAME before eliminating important variables when the $\PE$ drops to an unacceptable value. In this experiment, however, we also allowed FLAME to continue to drop variables until 28 variables were dropped and there were no remaining possible matches. Figure \ref{fig:irrel} compares estimated and true treatment effects for FLAME (Early Stopping), FLAME (Run Until No More Matches) and other methods. In this experiment, Figure \ref{subfig:irre-FLAME-early} is generated by stopping when the \PE\ drops (starting from values within [-2, -1]) to below -20, resulting in more than 15,000 matches out of 20,000 units. FLAME achieves substantially better performance than other methods, shown in Figure \ref{fig:irrel}.


 \begin{figure*}[h!]
\captionsetup[subfigure]{justification=centering}
\centering
     \subfloat[FLAME-early \label{subfig:irre-FLAME-early}]{%
       \includegraphics[width = 0.25\textwidth]{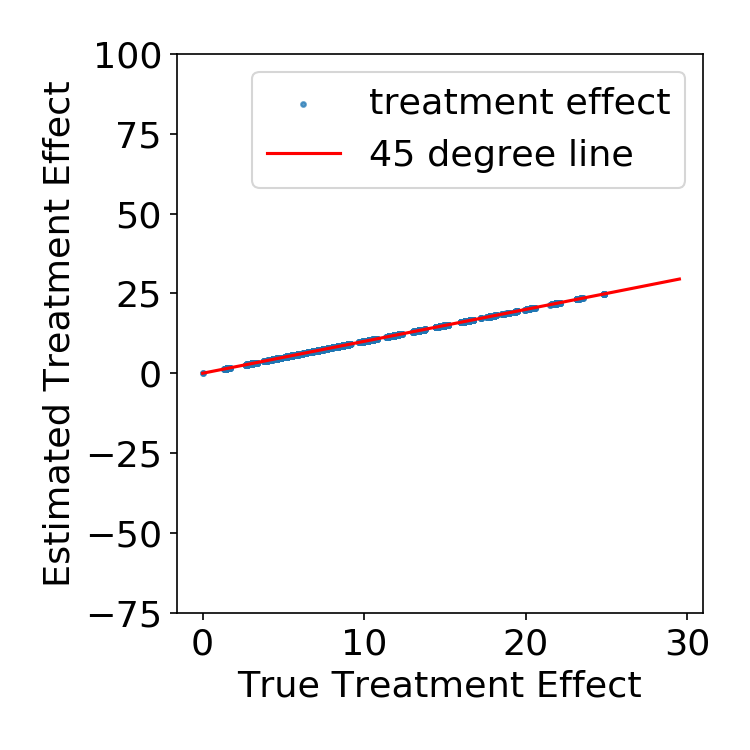}%
     }%
     \subfloat[ FLAME-NoMore\newline (Not recommended) \label{subfig:irre-FLAME-full}]{%
       \includegraphics[width = 0.25\textwidth]{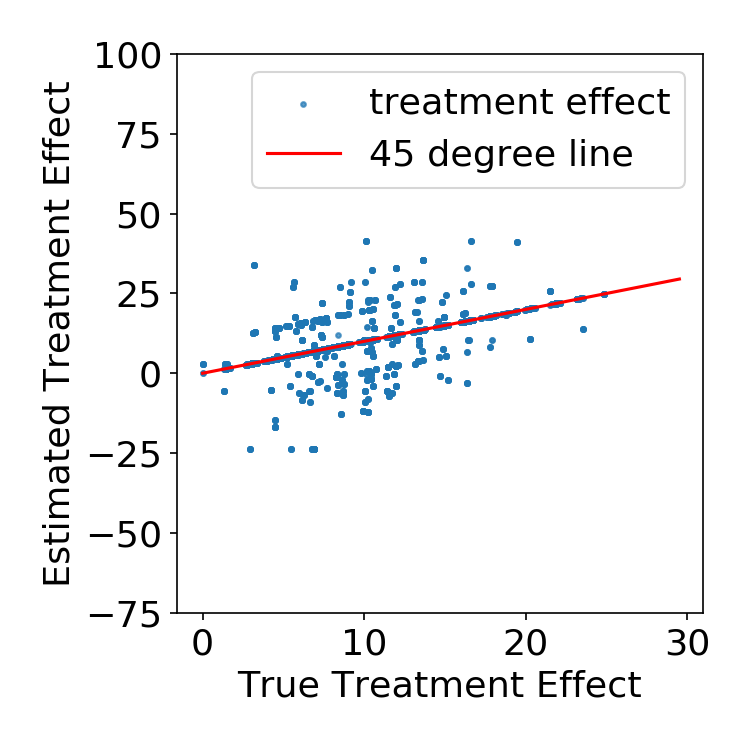}%
     }%
     \subfloat[GenMatch \label{subfig:irre-genmatch}]{%
       \includegraphics[width = 0.25\textwidth]{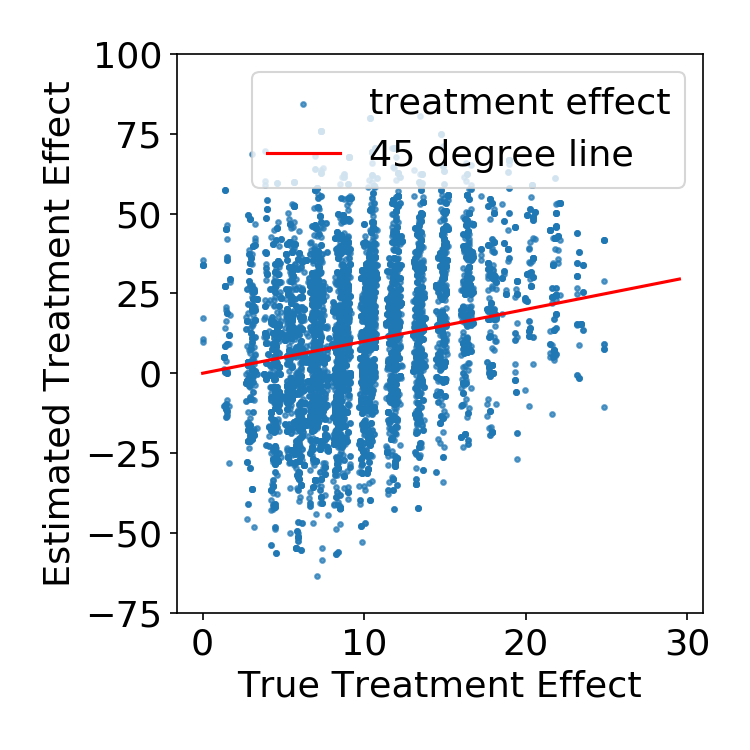} 
     } \\
     \subfloat[1-PSNNM \label{subfig:irre-propen}]{%
       \includegraphics[width = 0.25\textwidth]{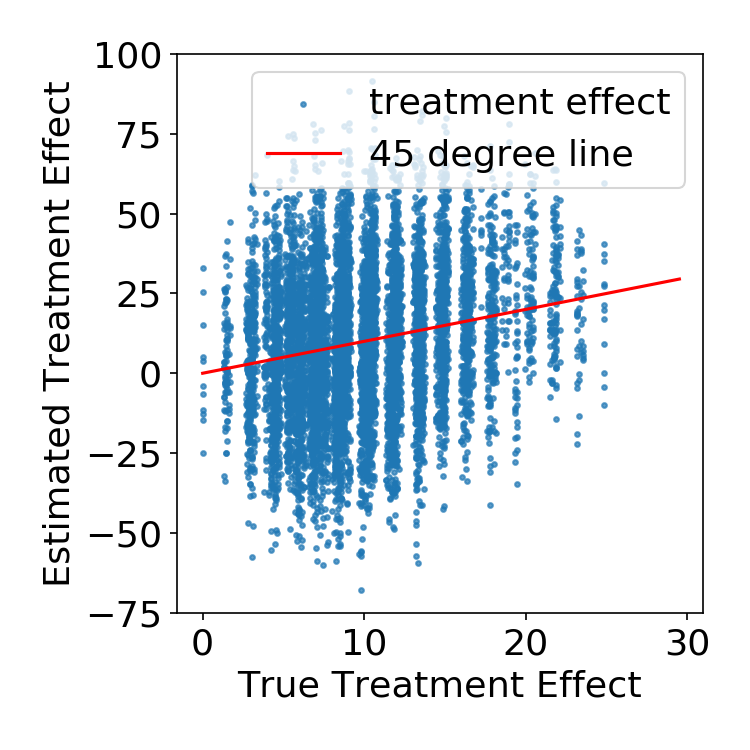}
     } 
     \subfloat[Oracle 1-PSNNM${}_a$ \label{subfig:Oracle-PSNNM-1}]{%
       \includegraphics[width = 0.25\textwidth]{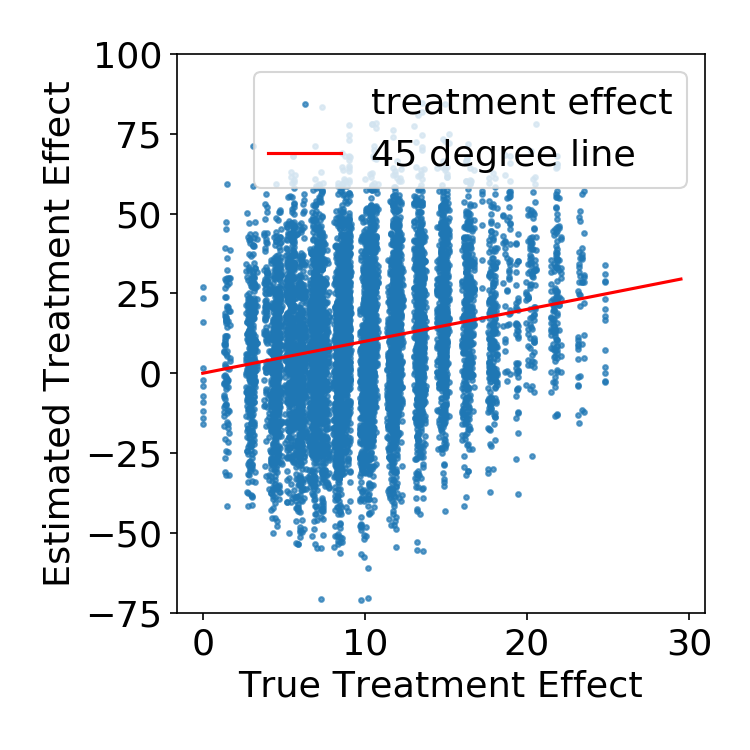}
     } 
     \subfloat[Oracle 1-PSNNM${}_b$ \label{subfig:Oracle-PSNNM-2}]{%
       \includegraphics[width = 0.25\textwidth]{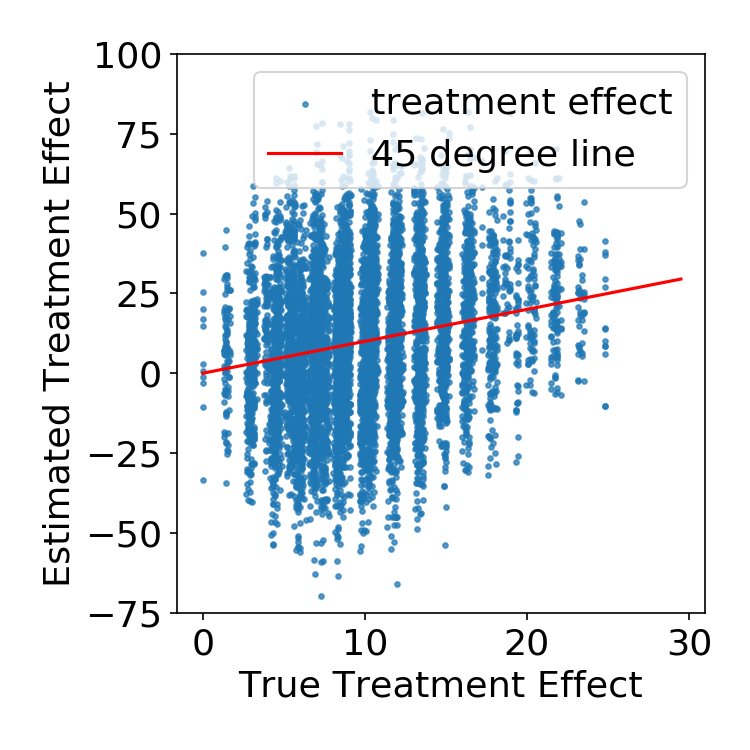} 
     } \\
     \subfloat[Mahalanobis \label{subfig:irre-nn}]{%
       \includegraphics[width = 0.23\textwidth]{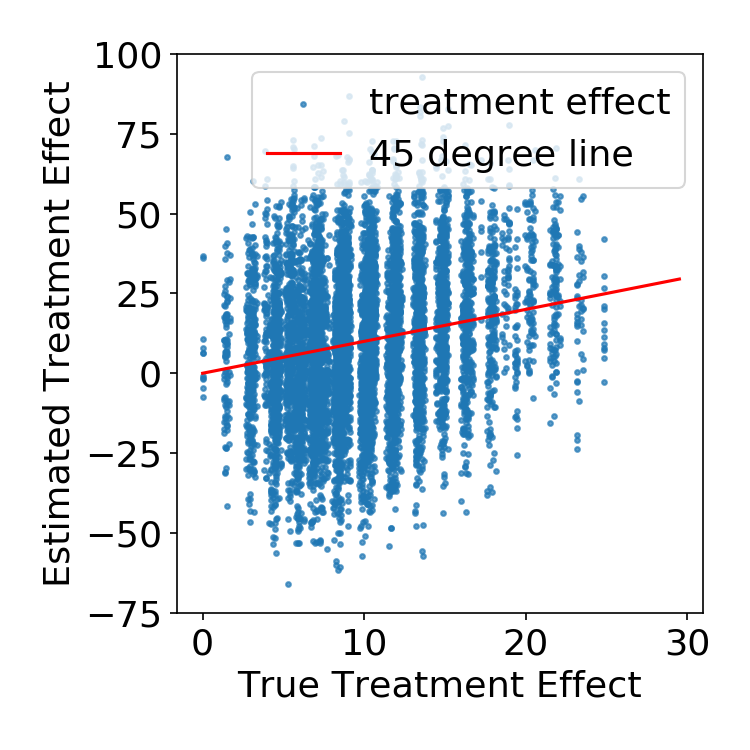}
     } 
     \subfloat[Causal Forest \label{subfig:irre-causalforest}\newline]{%
       \includegraphics[width = 0.23\textwidth]{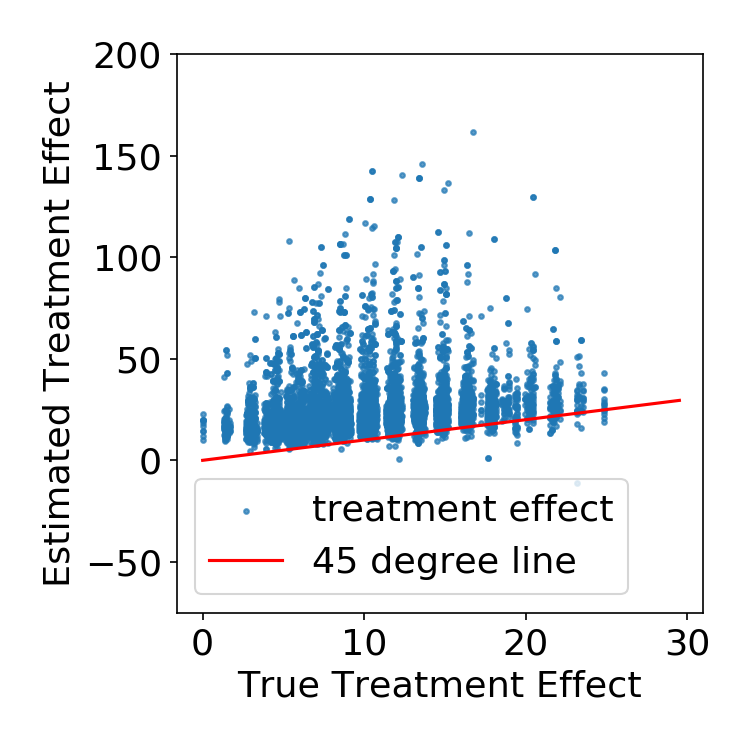}
     } 
    \subfloat[BART \label{subfig:barts}]{%
       \includegraphics[width = 0.23\textwidth]{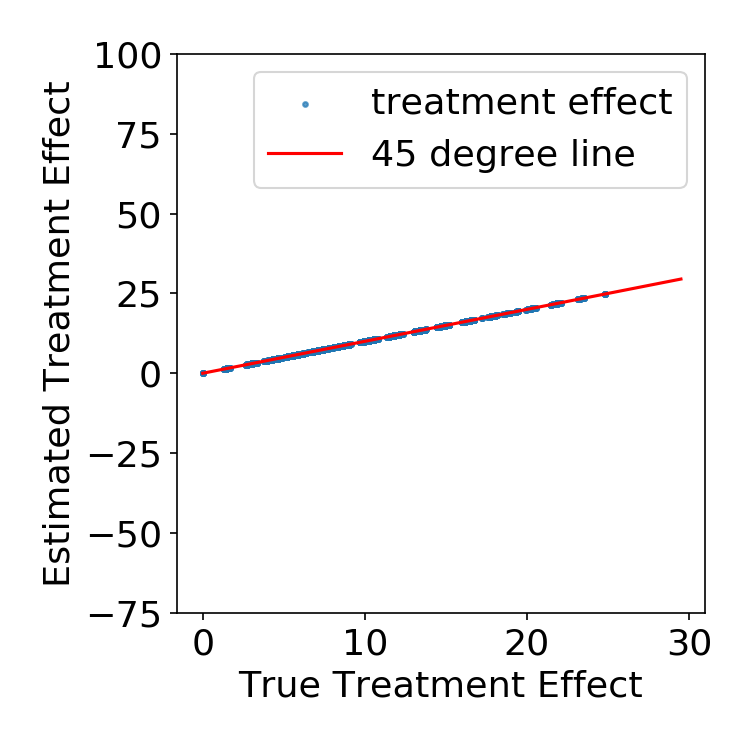}
     } 
     \subfloat[CTMLE \label{subfig:ctmle}]{%
       \includegraphics[width = 0.23\textwidth]{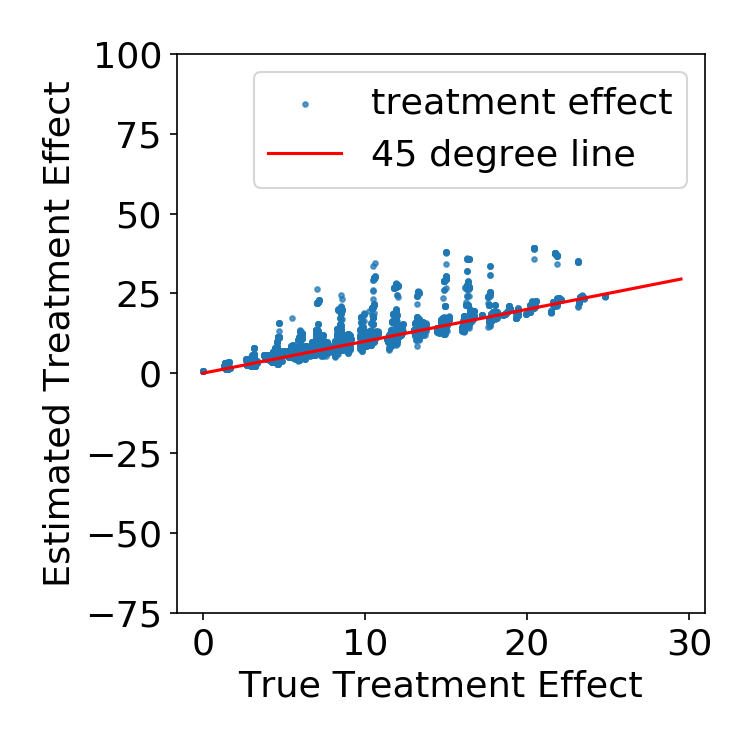}
     } 

     \caption{Scatter plots of true treatment effect and estimated treatment effect on matched units. For FLAME-early (\ref{subfig:irre-FLAME-early}), we extract the results before the \PE\ drops from [-2, -1] to below -20, and we call this early stopping. Before early stopping, we lose less than 2\% in \PE\ by removing covariates. In contrast, FLAME-NoMore (\ref{subfig:irre-FLAME-full}) is the result for running FLAME until there are no more matches. We do not recommend running FLAME past its stopping criteria but we did this for demonstration purposes. 
     CTMLE uses BART as its base predictor. Oracle 1-PSNNM${}_a$ runs 1-PSNNM with important covariates (the 10 covariates in Eq.~\ref{syntheticdata}), and Oracle 1-PSNNM${}_b$ runs 1-PSNNM with unimportant covariates (all of the 20 covariates that do not contribute to outcomes but influence the propensity score). The latter setting covers the ideal scenario for the propensity score matching with variable reweighting/selection based on propensity scores \citep[e.g.,][]{schneeweiss2009high}, since all variables that affect the propensity score are included in the model.
     }
     \label{fig:irrel} 
 \end{figure*}

This experiment illustrates the main issues with classes of methods for causal inference: propensity score matching (of any kind) projects the data to one dimension, and thus cannot be used for CATE estimation. GenMatch has similar problems, and cannot be used for reliable CATE estimation. Regression and other modeling methods are subject to misspecification. Most methods do not produce interpretable matches, and those that do cannot scale to datasets of the sizes we consider, as we will soon see. 



\subsubsection{Scalability evaluation}
\label{subsec:exp-scalability}

We compare the execution time of FLAME-bit and FLAME-db with the other methods. FLAME-bit is faster than most of the other approaches (including FLAME-db) for the synthetic data considered in this experiment. However, for larger dataset, FLAME-db is much faster than FLAME-bit and all other methods. Table~\ref{tab:compare-timing-census} compares the runtime when using the the US Census 1990 dataset \citep{Lichman:2013} with more than 1.2 million units and 59 covariates. Table \ref{tab:compare-timing-synthetic} summarizes the runtime of all methods when using synthetic data.


\begin{table}[ht]
\centering
\subfloat[Timing results of 
different methods on the US Census 1990 dataset \citep{Lichman:2013}.\label{tab:compare-timing-census}]{\begin{tabular}{ | c | c | }
  \hline
   Method & Time (hours) \\ \hline \hline
  FLAME-bit & Crashed \\ \hline 
  FLAME-db & 1.33 \\ \hline 
  Causal Forest & Crashed \\ \hline 
  1-PSNNM  & $>10$ \\ \hline 
  Mahalanobis & $>10$  \\ \hline 
  GenMatch& $>10$ \\ \hline
  Cardinality Match & $>10$  \\ \hline 
\end{tabular}}
\hspace{10pt}
\subfloat[Timing results of FLAME compared with other methods on the synthetic data generated by the same procedure as in Figure \ref{fig:irrel}, in the format of ``average time'' $\pm$ ``standard deviation.'' It summarizes 3 runs.  \label{tab:compare-timing-synthetic}]{
\begin{tabular}{ | c | c | }
  \hline
   Method & Time (seconds) \\ \hline \hline
  FLAME-bit & $22.30\pm0.37$ \\ \hline 
  FLAME-db & $59.68\pm0.24$ \\ \hline 
  Causal Forest & $52.65 \pm 0.41$ \\ \hline 
  1-PSNNM & $13.88 \pm 0.14$  \\ \hline 
  Mahalanobis & $55.78 \pm0.14$ \\ \hline 
  GenMatch & $>150 $  \\ \hline 
   Cardinality Match & $>150$  \\ \hline 
\end{tabular}}
\caption{Timing results for FLAME.}
\end{table}

\subsubsection{Upper bound on CATE sample variance}
While the main thrust of the methodology is in reducing the bias of estimating a CATE, the methodology naturally lends itself for estimating an upper bound on the variance of the CATE.
Here we briefly discuss the conditional variance of the estimation in the setting of Figure \ref{fig:irrel}. We note that asymptotic variance estimates for ATEs may also be obtained with the methods outlined in \cite{abadie2012martingale}.

While the covariance between the potential  outcomes is not able to be calculated directly, we can upper bound it using the Cauchy-Schwarz inequality to give us: 
\begin{align}
\text{Var} ( y_t - y_c | \mathbf{x}) \le \left( \text{Std} ( y_t  | \mathbf{x}) + \text{Std} ( y_c | \mathbf{x}) \right)^2. \label{eq:var-bound}
\end{align}
We note that if it is reasonable to assume that the potential outcomes are not negatively correlated, a tighter upper bound is possible as the cross-term can be ignored. 
We then use the sample variance to estimate $\text{Std} ( y_t | \mathbf{x})$ and $\text{Std} ( y_c | \mathbf{x})$ for each group. 
A simulated study of variance upper bound is shown in Figure \ref{fig:std}, where the simulation data is generated using Eq. \ref{syntheticdata} (with 25,000 treated units, 25,000 control units, 20 irrelevant covariates, variance of Gaussian noise being 1, second order coefficient $U=1$, and all other parameters as default).
Each box plot 
 in the figure summarizes the standard deviations of groups in a single iteration of the \textit{while} loop. They are arranged according to the order in which the variables are dropped. 
This method estimates a small standard deviation for the first few levels as no relevant covariates have been eliminated. The standard deviation increases abruptly once a relevant covariate is dropped since this introduces variability in the treatment effects observed within each group. A sharp increase in the standard deviations thus likely corresponds to the dropping of important covariates and a decay in the match quality. This suggests another heuristic for the early stopping of FLAME, which is to stop when a sharp increase is observed.  
\begin{figure}[ht]
\centering
       \includegraphics[width = 0.4\textwidth]{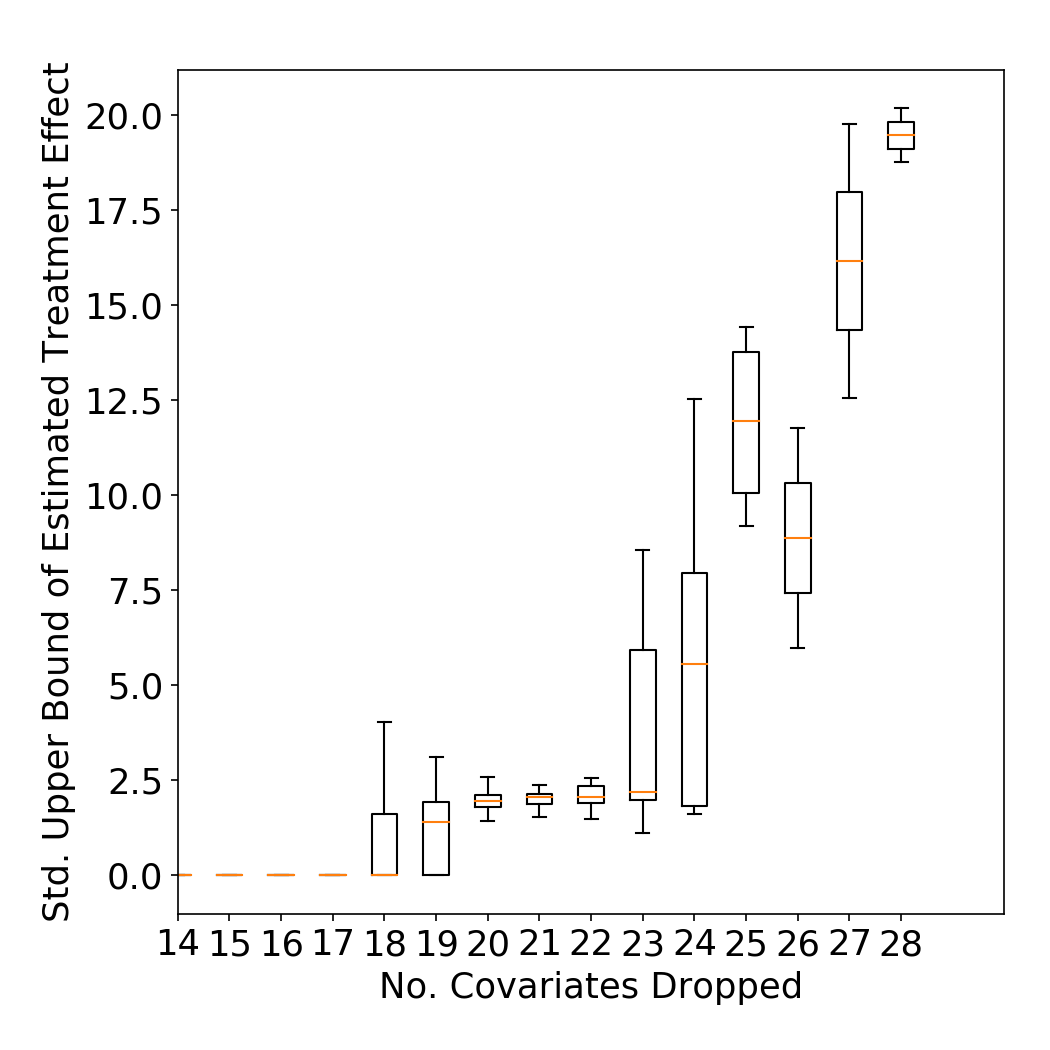}
     \caption{This figure plots the standard deviation upper bound estimates, i.e., empirical samples of square root of right-hand-side of Eq.~\ref{eq:var-bound}. The box plots summarize the estimated standard deviation upper bounds of the groups created as covariates are eliminated. The orange lines in the boxes are medians. The lower and upper boundaries of the boxes are the first quartile and the third quartile of the distribution, while the lower and upper black bar correspond to 5th and 95th percentile, respectively. 
          \label{fig:std} } 
\end{figure}

\subsection{Experiments with Real Data}
Here we use FLAME to obtain estimates of treatment effects for an important societal question: the effect of prenatal smoking on the health of the newborn child. 


\begin{figure}[ht]
    \centering
    \includegraphics[width = 0.5\textwidth]{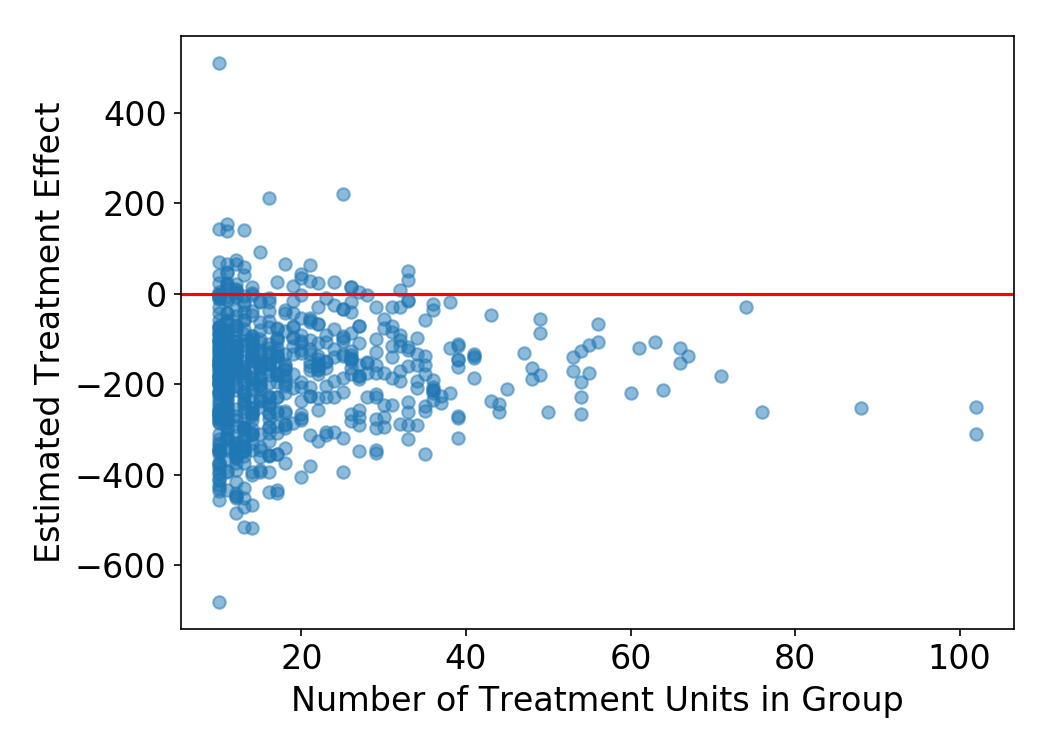}
    \caption{Scatter plots of estimated treatment effect (on newborn's weight in grams) against group size. Only matched groups with more than 30 control units and 10 treatment units are shown. The red line plots treatment effect being $ 0$.
    \label{fig:natality-weight}
    }
\end{figure}

There is a large literature studying the effects of maternal smoking during pregnancy on neonatal health outcomes. For example, recent work by \citet{kondracki2020low} shows a significant effect of ``extreme smoking'' (defined as smoking at least 10 cigarettes daily throughout the pregnancy) on the odds of an infant requiring a Neonatal Intensive Care Unit (NICU) stay. In this section we employ FLAME to yield a more granular understanding of such claims---endeavoring to understand whether heterogeneous effects might be present. To do this, we study the 2010 US Natality dataset from the National Vital Statistics System of the \citet{natality2010}. As done by \citet{kondracki2020low}, we restrict ourselves to standard term (37+ weeks) singular (i.e., not twins, triplets, etc.) births. Treatment is defined as smoking at least 10 cigarettes per day for the duration of the pregnancy and our comparison group consists of those women who did not smoke at all during pregnancy \citep{kondracki2019prevalence}. The demographic and medical variables we consider can be found in Table \ref{tab:drop-order} in the appendix. The data we were provided with includes discretized mother's and father's age variables, binned at five year increments, that we use for analysis. An alternative left for follow up work can leverage tools for continuous variables such as MALTS \citep{MaltsObsStudies} in conjunction with FLAME to analyze the raw data, including the raw age.

We study two topics: \textit{NICU admissions} and \textit{child's low birth weight}. Both demonstrate the strength of our approach in providing granular causal analyses, but in different ways. After eliminating units whose outcomes and/or treatment indicators are missing, there were $\sim$2.1M units, among which $\sim$75K units are treated units. We first estimate an overall odds ratio (by averaging odds ratio across matched groups) of NICU admission of approximately 2.6 which is in line with the literature \citep{kathleen2002neonatal}. The power of our methodology is the ability to perform a deeper dive into the data. We identify multiple large groups with odds ratios that are substantially below 1, shown in Figure \ref{fig:natality-preterm} in the appendix, which suggests that \textit{the available data are not sufficient for granular analysis, or any strong conclusion on NICU admissions}, since it is unlikely that there is a protective effect of smoking on NICU admissions \citep{abrevaya2006estimating, honein2001maternal}. These results can potentially be explained by the diverse set of reasons that individuals might be admitted into the NICU that are not accounted for in these data \citep{mahendra2020predicting}. Such issues with these data have not (as far as we know) previously been identified despite the fact that they impact the trustworthiness of results on this important topic \citep{kathleen2002neonatal}, where this same dataset is often used as the trusted source for studying this topic. 

On the other hand, our analysis of the effect of ``extreme smoking'' on low birth weight leads to substantially clearer conclusions: The estimated average treatment effect of ``extreme smoking'' on birth weight is $-248$ grams of infant's weight, and the CATEs are presented for the high quality matched groups in Figure~\ref{fig:natality-weight}. We note that the estimates for the largest matched groups are most reliable and are also consistent with the literature \citep{kataoka2018smoking,kondracki2020low}. Importantly, this plot suggests little to no heterogeneity in the conditional treatment effects of smoking on birth weight; most of the extreme values are observed among small matched groups for which estimation is most difficult.

In summary, FLAME's ability to perform granular analysis is a determiner of trust: in some cases (e.g., smoking causing NICU admissions) we might question the results from other types of methods, e.g., those from methods that output an overall ATE. In other cases, FLAME might solidify our trust in the results (e.g., smoking causing lower birth weight). FLAME's results on NICU admission calls for (1) further medical study on the causal effect of prenatal smoking on NICU admission with a finer segmentation of the population, and (2) consideration of socio-economical causes (e.g., NICU care may not be affordable/available in certain circumstances) in order to obtain a cleaner medical causal relation between smoking and NICU admissions. Without the granular analysis that FLAME provides, it is possible that, in this analysis, as well as in many other analyses in other fields, incorrect/untrustworthy conclusions may not be noticed or investigated.

\section{Conclusion}\label{sec:conclusions}
FLAME produces interpretable, high-quality matches. The defining aspects of FLAME are that it (i) \textit{learns} an \textit{interpretable} distance for matching by leveraging a training set, and (ii) uses an algorithm that eliminates successively less-relevant covariates, and permits SQL and bit-vector operations to scale to millions of observations -- these are datasets so large that they may not fit into main memory. We were able to theoretically evaluate the bias of FLAME under different distributions of the covariates. Through application to both simulated and real-world datasets, we were able to show that FLAME either ties with or outperforms existing matching methods in terms of estimation error, and provides a benefit in interpretability. 

There are several natural extensions to the FLAME procedure: First, the less greedy approach of the Dynamic Almost Matching Exactly (DAME) algorithm \citep{DiengEtAl2019} extends the work in this paper by providing a solution to the FULL-AME problem, without backwards selection. However, DAME comes with a substantial computational burden as compared to FLAME and so can only be realistically applied to moderately sized problems. \citet{DiengEtAl2019} demonstrate that in many practical settings FLAME and DAME achieve approximately the same accuracy for CATE estimation, and suggest that FLAME's backwards selection procedure be used until the number of covariates is small enough for a solution to DAME to be practical. 

An obvious shortcoming of the FLAME framework is that it operates on categorical covariates. If we are given substantive information on what changes to the covariates would not influence outcomes, one can first coarsen continuous covariates (as in CEM), but such information is not always available. 
One solution for mixed continuous/discrete data is to use interpretable ``stretch'' matrices as distance metrics, and match to nearest neighbors using this stretched distance metric. This idea is the foundation for the MALTS algorithm \citep{ParikhRuVo18,MaltsObsStudies}. MALTS builds on FLAME by using a  training set to learn the distance for matching. In doing this, \citet{ParikhRuVo18} suggest to combine FLAME's learned Hamming distance for discrete data with MALTS' interpretable stretch metric for real-valued data. A further extension of FLAME, called Adaptive Hyperboxes (AHB), learns an interpretable box individually (and nonparametrically) around each data point \citep{AME-AHB-2020}. Units within the box are matched. AHB relies on an underlying black box model to determine the distance metric, which is a key difference from FLAME and MALTS. 

FLAME has been extended to handle network interference \citep{AME-Networks-2020} and  instrumental variable analysis \citep{AME-IV-2019}.

The code for FLAME is available at \url{https://cran.r-project.org/web/packages/FLAME/index.html} (in R), and \url{https://www.github.com/almost-matching-exactly/} (in Python). An introduction to the project with links to the code is also found at
\url{https://almost-matching-exactly.github.io/}. \citet{gupta2021dameflame} provides a short overview of the FLAME-DAME software package.

 \subsection*{Acknowledgments}
We thank Jerry Chang for coding early versions of the FLAME ecosystem and to Neha Gupta and Vittorio Orlandi for the production versions of the Python and R packages. We would also like to thank the other members of the Almost Matching Exactly Lab at Duke for invaluable help: Harsh Parikh, Xian Sun, Thomas Howell, and Angikar Ghosal. We would like to thank Zijun Gao and Guillaume Basse for helpful discussions. 
This work was partially supported by NIH QuBBD award R01EB025021, NSF award IIS-1703431, and by the Duke Energy Initiative.

\bibliography{biblio} 
\clearpage

\part*{Appendix}
\renewcommand{\thesection}{\Alph{section}}
\setcounter{section}{0}

\section{Proof of Proposition \ref{thm:bound}} \label{app:proof-bias-bound}

For a set of $i = 1, \dots, n$ units, consider potential outcomes $g^{(t)}(\x_i)$, with covariates $\x \in \{0, 1\}^p$ and treatment indicator $t \in \{0, 1\}$. Observed treatments are denoted by the random variable $T_i$ for all units $i$, and observed outcomes are denoted by $Y_i = T_ig^{(1)}(\x_i) + (1-T_i)g^{(0)}(\x_i)$.  Let $\vneq{\x}{\x'}$ be a vector that is one at position $j$ if $x_j \neq x_j'$ and 0 everywhere else. Let $\bw^T\vneq{\x}{\x'}$ be the (positive) weighted Hamming distance between $\x$ and $\x'$ with a vector of weights of length $p$, denoted by $\bw$, where each entry of $\bw$ is a positive real number, and such that $0 <\|\bw\|_2 <\infty$. Define the matched group for a value $\x$ as $\MGS = \{ \x_i, \; i = 1,\dots, n\; s.t:\; \x_i \circ \theta = \x \circ \theta\}$, i.e., the set of units in the sample that have value $\x$ of the covariates selected by $\theta$. Define also $\nxt = \sum_{\x_i \in \MGS}T_i$ and $\nxc = \sum_{\x_i \in \MGS}(1-T_i)$. Let: $$M := \max_{\substack{ \x, \x' \in \{0, 1\}^p \\t \in \{0, 1\}}}\frac{|g^{(t)}(\x') - g^{(t)}(\x)|}{\bw^T\vneq{\x'}{\x}},$$
and assume that $M < \infty$. Then, for all $\x, \x'\in\{0, 1\}^p$ and $t \in \{0, 1\}$, by the definition of $M$ above, we have: $|g^{(t)}(\x) - g^{(t)}(\x')| \leq M\bw^T\vneq{\x'}{\x}$ and 
\begin{align}
    g^{(t)}(\x) - M \bw^T\vneq{\x}{\x'} \leq g^{(t)}(\x')\leq g^{(t)}(\x) + M\bw^T\vneq{\x}{\x'}. \label{eq:g-abs-bound}
\end{align}
Now we pick an arbitrary pair of $\x$ and $\btheta$ and consider the term $\bw^T\vneq{\x}{\x'}$ for any $\x' \in \MGS$:
\begin{align}
    \bw^T\vneq{\x}{\x'} &= \sum_{j=1}^pw_j\vneq{x_j}{x_j'} \com{$x_j$(resp. $x_j'$) is the $j$-th entry of $\x$ (resp. $\x'$)}\nonumber\\
    &= \sum_{j=1}^pw_j(1-\theta_j)\vneq{x_j}{x_j'}\com{Since $ \x' \in \MGS$}\nonumber\\
    &\leq \sum_{j=1}^pw_j(1-\theta_j) = \bw^T(\bone - \btheta), \label{Eq:bound-distance-sum}
\end{align}
where $\bone$ is a vector of length $p$ that is one in all entries. The second line follows because $\x' \in \MGS$ implies that $\x'$ must match exactly with $\x$ on the covariates selected by $\btheta$. 
We want to use the estimator $\hat{\tau}(\x, \btheta) = \frac{1}{\nxt}\sum_{\x_i \in \MGS}Y_iT_i - \frac{1}{\nxc}\sum_{\x_i \in \MGS}Y_i(1-T_i)$ to estimate $\tau(\x) = g^{(1)}(\x) - g^{(0)}(\x)$ for some fixed $\x$ and $\btheta$. 
We can construct an upper bound on the estimation error of this estimator as follows:
\begin{align}
    &\frac{1}{\nxt}\sum_{\x_i \in \MGS}Y_iT_i - \frac{1}{\nxc}\sum_{\x_i \in \MGS}Y_i(1-T_i) - \tau(\x) \nonumber \\
    =& \frac{1}{\nxt}\sum_{\x_i \in \MGS}g^{(1)}(\x_i)T_i \nonumber\\
    &- \frac{1}{\nxc}\sum_{\x_i \in  \MGS}g^{(0)}(\x_i)(1-T_i) - (g^{(1)}(\x) - g^{(0)}(\x)), \label{eq:bound-intermediate-1}
\end{align} 
where we use $Y_i = T_i g^{(1)} (\x_i ) +  (1 - T_i) g^{(0)} (\x_i )$, and $T_i^2 = T_i$ and $T_i (1 - T_i) = 0$ for any $i$. Then from (\ref{eq:g-abs-bound}) we get:
\begin{align}
    &\frac{1}{\nxt}\sum_{\x_i \in \MGS}g^{(1)}(\x_i)T_i\nonumber\\
    &-\frac{1}{\nxc}\sum_{\x_i \in \MGS}g^{(0)}(\x_i)(1-T_i) - (g^{(1)}(\x) - g^{(0)}(\x)) \nonumber \\
    \leq& \frac{1}{\nxt}\sum_{\x_i \in \MGS}[g^{(1)}(\x) + M\bw^T\vneq{\x}{\x_i}]T_i \nonumber\\
    &- \frac{1}{\nxc}\sum_{\x_i \in \MGS}[g^{(0)}(\x) - M\bw^T\vneq{\x}{\x'}](1-T_i) \nonumber \\
    & - g^{(1)}(\x) + g^{(0)}(\x) .
    \label{eq:bound-intermediate-2}
\end{align}
Next, we combine (\ref{eq:bound-intermediate-1}) and (\ref{eq:bound-intermediate-2}) to get:
\begin{align*}
    & \frac{1}{\nxt}\sum_{\x_i \in \MGS}Y_iT_i - \frac{1}{\nxc}\sum_{\x_i \in \MGS}Y_i(1-T_i) - \tau(\x) \\
    \leq& \frac{1}{\nxt}\sum_{\x_i \in \MGS}[g^{(1)}(\x) + M\bw^T\vneq{\x}{\x_i}]T_i \nonumber \\
    &- \frac{1}{\nxc}\sum_{\x_i \in \MGS}[g^{(0)}(\x) - M\bw^T\vneq{\x}{\x'}](1-T_i) \nonumber \\
    & - g^{(1)}(\x) + g^{(0)}(\x) \\
    =& g^{(1)}(\x)\frac{1}{\nxt}\sum_{\x_i \in \MGS}T_i \\
    &- g^{(0)}(\x)\frac{1}{\nxc}\sum_{\x_i \in \MGS}(1-T_i) - g^{(1)}(\x) + g^{(0)}(\x)\\
    &+ \sum_{\x_i \in \MGS}M\bw^T\vneq{\x}{\x_i}\left(\frac{T_i}{\nxt} + \frac{1-T_i}{\nxc}\right)\\
    =& \sum_{\x_i \in \MGS}M\bw^T\vneq{\x}{\x_i}\left(\frac{T_i}{\nxt} + \frac{1-T_i}{\nxc}\right)\\
    \leq& \sum_{\x_i \in \MGS }M\bw^T(\bone - \theta)\left(\frac{T_i}{\nxt} + \frac{1-T_i}{\nxc}\right)\\
    =& 2M\bw^T(\bone - \btheta),
\end{align*}
where the inequality in the second to last line follows from Equation \eqref{Eq:bound-distance-sum}. {Using the same set of steps we can construct a lower bound on the estimation error:}
\begin{align*}
    &\frac{1}{\nxt}\sum_{\x_i \in \MGS}Y_iT_i - \frac{1}{\nxc}\sum_{\x_i \in \MGS}Y_i(1-T_i) - \tau(\x)\\
    \geq& \frac{1}{\nxt}\sum_{\x_i \in \MGS}[g^{(1)}(\x) - M\bw^T\vneq{\x}{\x_i}]T_i \\
    &- \frac{1}{\nxc}\sum_{\x_i \in \MGS}[g^{(0)}(\x) + M\bw^T\vneq{\x}{\x'}](1-T_i) \\
    & - g^{(1)}(\x) + g^{(0)}(\x)\\
    \geq& - \sum_{\x_i \in \MGS }M\bw^T(\bone - \theta)\left(\frac{T_i}{\nxt} + \frac{1-T_i}{\nxc}\right)\\
    =& - 2M\bw^T(\bone - \btheta).
\end{align*}
Putting together these two bounds we obtain the statement in the proposition. 


\section{Proof of Theorems in Section \ref{sec:theory}}
\begin{proof}\textit{\bf of Theorems \ref{thm:twocov} and \ref{thm:threecov}:} 
The proofs of these theorems are computer programs. They calculate the bias analytically (rather than numerically). They enumerate all possible covariate allocations. For each allocation, the program checks if the allocation is valid and then runs Oracle FLAME to compute the conditional average treatment effect in each bin. Subtracting these estimate from the true CATE values and averaging (uniformly) over bins yields the bias for that allocation. The theorems report the average over valid allocations. 
\end{proof}

\begin{proof}\textit{\bf of Theorem \ref{thm:noalpha}:}
The results follows from the symmetry in the assignment of treatments and controls to valid allocations. Note that if an allocation to $n$ units $T_n =(t_1,\dots,t_n)$ is a valid allocation then so is $1-T_n$. 
If the estimate of a CATE under the assignment $T_n$ is $\sum \gamma_j\alpha_j + \sum \xi_j\beta_j$ then necessarily the estimate of the same CATE under $1-T_n$ is $- \sum \gamma_j\alpha_j + \sum \eta_j\beta_j$ for some $\eta_j$. Thus the contribution of baseline covariate $j$ to the estimation of the CATE is necessarily identical under $T_n$ and $1-T_n$ since it does not depend on the treatment value of any unit. Let us define $\widehat{CATE}_X(T_n)$ as the conditional average treatment effect for covariate combination $X$ under assignment $T_n$. Note that 
the bias of estimating any CATE can be written as 
\begin{equation*}
    \frac{1}{ 2^{n} } \left( \sum_{T_n :\text{unit }1 \text{ treated}} \widehat{CATE}_X(T_n)+\sum_{T_n:\text{unit }1\text{ control}} \widehat{CATE}_X(T_n) \right) - CATE_X 
\end{equation*}
where the summations are over all possible assignments. For every assignment $T_n$ in the first summand, $1-T_n$ must appear in the second summand, thus canceling the contribution of all $\alpha_j$ to the bias.
\end{proof}

\begin{proof}\textbf{of Proposition \ref{prop:empirical-bias}:}
We focus on the treated case, and the control counterpart follows naturally using the same argument. Since $ p_1 (X \circ \btheta) \ge \lambda_{\left\| \btheta \right\|_0} $, the probability of the error maximizing point of a fixed $f^{(1)} \in \mathcal{F}_{\left\| \btheta \right\|_0} $ not in $\mathcal{S}_1^{ma}$ is at most $ (1 - \lambda_{\left\| \btheta \right\|_0 + r })^{ | \mathcal{MG} (\btheta, \mathcal{S}_1^{ma}) | }  $. This is true because the max error for a given $f^{(1)}$ only depends on the covariates selected by $\btheta$ and the $r$ relevant covariates that contribute to the outcome (in the noiseless setting). Similarly, the probability that the error maximizing point not in $\mathcal{S}_1^{tr}$ is at most $ (1 - \lambda_{\left\| \btheta \right\|_0 + r })^{ | \mathcal{MG} (\btheta, \mathcal{S}_1^{tr}) | }  $. Thus with probability at least $ 1 - (1 - \lambda_{\left\| \btheta \right\|_0 + r })^{ |  \mathcal{S}_1^{tr} | } - (1 - \lambda_{\left\| \btheta \right\|_0 + r })^{ | \mathcal{MG} (\btheta, \mathcal{S}_1^{ma}) | }$, the error maximizing point for $f^{(1)}$ is in both $\mathcal{MG} (\btheta, \mathcal{S}_1^{tr}) $ and $ \mathcal{S}_1^{ma} $. When this happens, we have 
\begin{align}
    \left| \hat{g} (\mathbf{x}, \btheta, \mathcal{S}_1^{ma}) - g^{(1)} (\mathbf{x}) \right| &\le \left| \hat{g} (\mathbf{x}, \btheta, \mathcal{S}_1^{ma}) - f^{(1)} (  \mathbf{x} \circ \btheta )  \right| + \left| f^{(1)} ( \mathbf{x} \circ \btheta ) - g^{(1)} (\mathbf{x}) \right| \\
    &\le \hat{\epsilon}^{\max}_1 (\mathcal{S}_1^{ma}, f^{(1)}, \btheta) + \hat{\epsilon}^{\max}_1 (\mathcal{S}_1^{tr}, f^{(1)}, \btheta) \label{eq:max-error-1} \\
    &= 2 \hat{\epsilon}^{\max}_1 (\mathcal{S}_1^{tr}, f^{(1)}, \btheta), \label{eq:max-error-2}
\end{align}
where (\ref{eq:max-error-1}) and (\ref{eq:max-error-2}) are due to that both $\hat{\epsilon}^{\max}_1 (\mathcal{S}_1^{ma}, f^{(1}), \btheta)$ and $ \hat{\epsilon}^{\max}_1 (\mathcal{S}_1^{tr}, f^{(1)}, \btheta)$ in this case are true max error for $f^{(1)}$. 

\end{proof}

\section{Additional Experiment Details} 
\label{app:add-exp}

\subsection{Regression cannot handle model misspecification} 
We generated 20,000 observations from (\ref{syntheticdata}), with 10,000 treatment units and 10,000 control units, where  $U$=10 and $x_i^T,x_i^C \sim$ Bernoulli$(0.5)$. The nonlinear terms will cause problems for the (misspecified) linear regression models, but matching methods generally should not have trouble with nonlinearity. We ran FLAME, and all points were matched exactly on the first iteration yielding perfect CATEs. Scatter plots from FLAME and regression are in Figure \ref{fig:misspec}. The axes of the plots are predicted versus true treatment effects, and it is clear that the nonlinear terms negatively impact the estimates of treatment effects from the regression models, whereas FLAME does not have problems estimating treatment effects under nonlinearity. 

\begin{figure}[h!]
\centering
     \subfloat[FLAME\label{subfig:quality-FLAME}]{%
       \includegraphics[width = 0.4\textwidth]{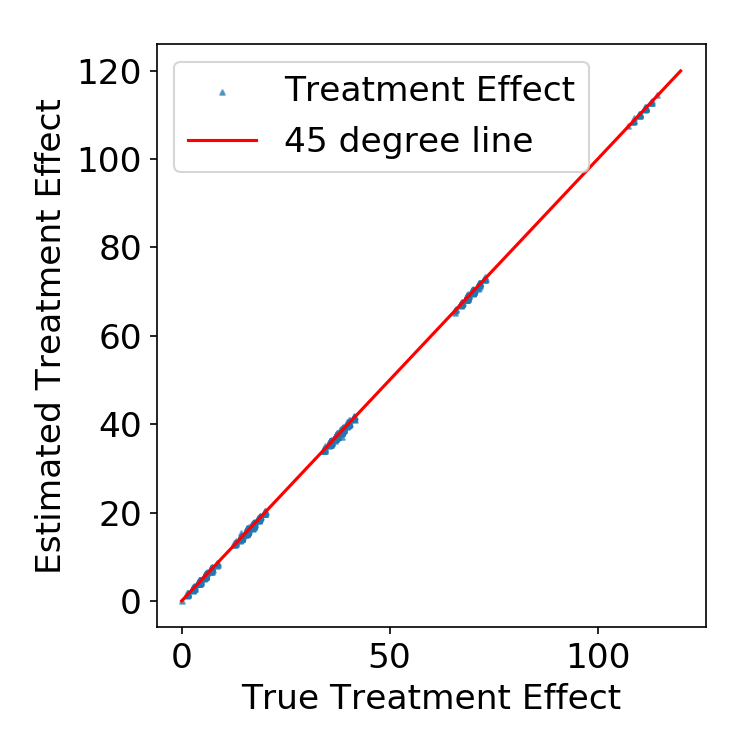}
     }
     \subfloat[Double linear regressors\label{subfig:quality-reg}]{%
       \includegraphics[width = 0.4\textwidth]{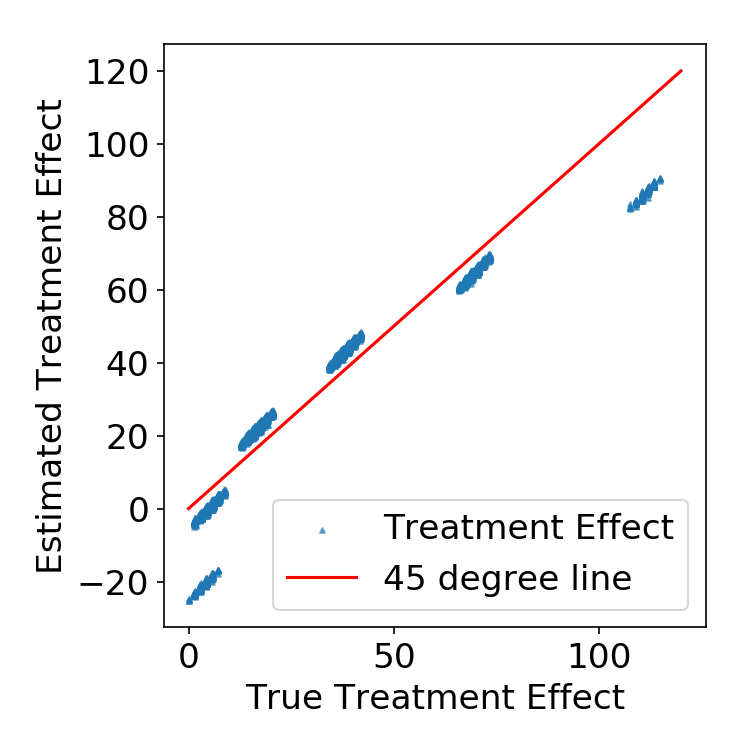}
     } 
     \caption{Scatter plots of true treatment effect versus estimated treatment effect on every unit in the synthetic data. The regression model is misspecified, and performs poorly. }
     \label{fig:misspec} 
\end{figure}

\subsection{Scalability Comparison between FLAME-bit and FLAME-db} \label{app:scalability-bit-db}

In this part, we evaluate scalability of both implementations of FLAME (FLAME-db and FLAME-bit) in a more complete manner. We generate synthetic datasets using the method described in (\ref{syntheticdata}) for different values of $n$ and/or $p$. As in the previous settings, $x^C_i \sim \text{Bernoulli}(0.1)$ and $x^T_i \sim \text{Bernoulli}(0.9)$ for $i > 10$. 


We compare the run time of FLAME-bit and FLAME-db as functions of $p$ in Figure~\ref{subfig:timing-covs} (with $n$ fixed to 100,000), and of $n$ in Figure~\ref{subfig:timing-units} (with $p$ fixed to 15). 
In Figure \ref{fig:timing}, 
each dot represents the average runtime over a set of 4 experiments with the same settings; the vertical error bars represent the standard deviations in the corresponding set of experiments (we omit the very small error bars). 
The plots suggest that FLAME-db scales better with the number of covariates (though that happens more often when the number of covariates is large, beyond what we show here), whereas FLAME-bit scales better with the number of units, since pre-processing and matrix operations used in FLAME-bit are expensive when the number of columns is large.

A major advantage of FLAME-db is that it can be used on datasets that are too large to fit in memory, whereas FLAME-bit cannot be used for such large datasets.

\begin{figure}[ht]
\vspace{-0.3cm}
\centering   
    \subfloat[$n=100,000$ \label{subfig:timing-covs}]{%
\includegraphics[width=0.4\textwidth]{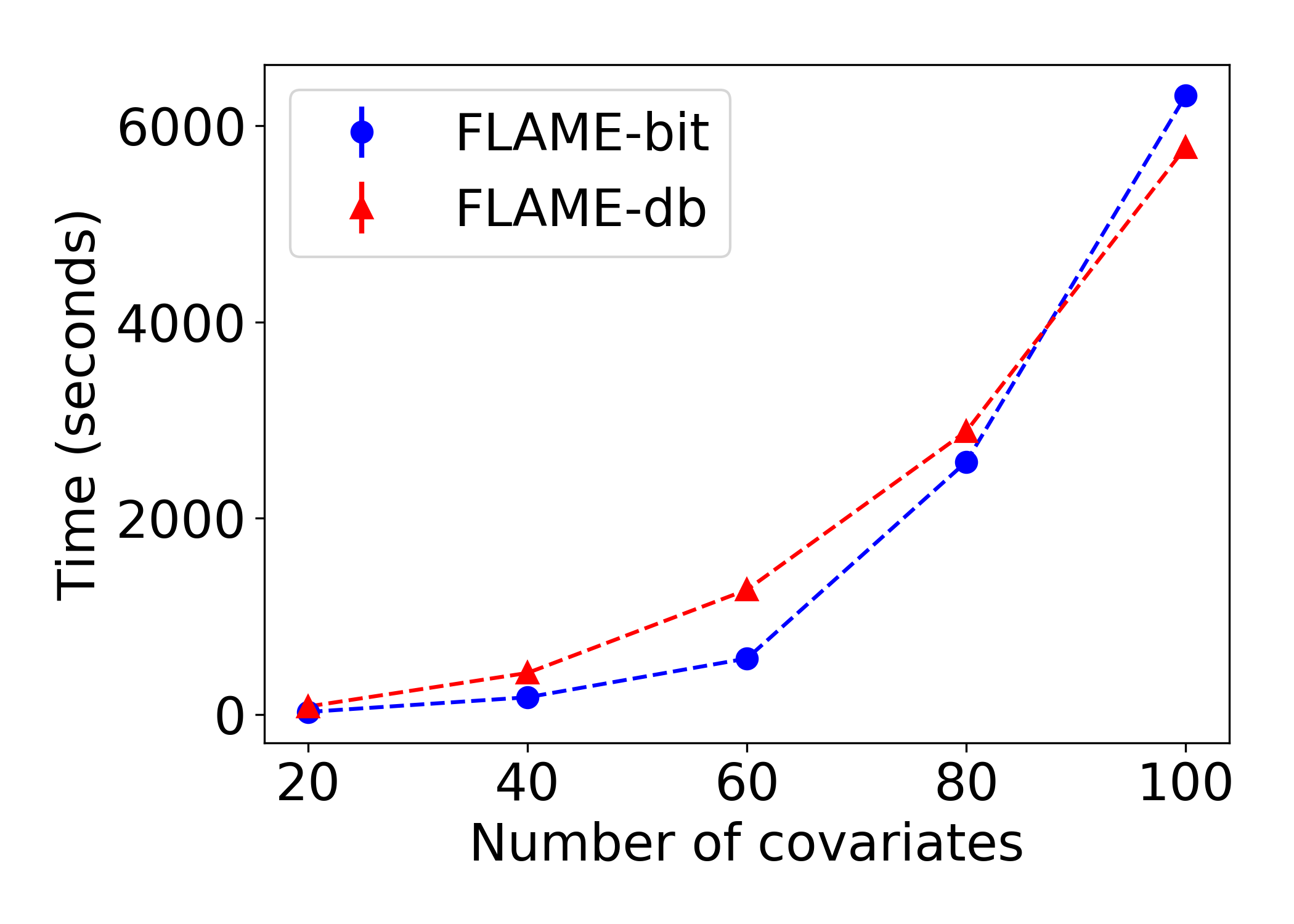}
     }
     \hspace{20pt}
     \subfloat[$p=15$ \label{subfig:timing-units}]{%
      \includegraphics[width=0.4\textwidth]{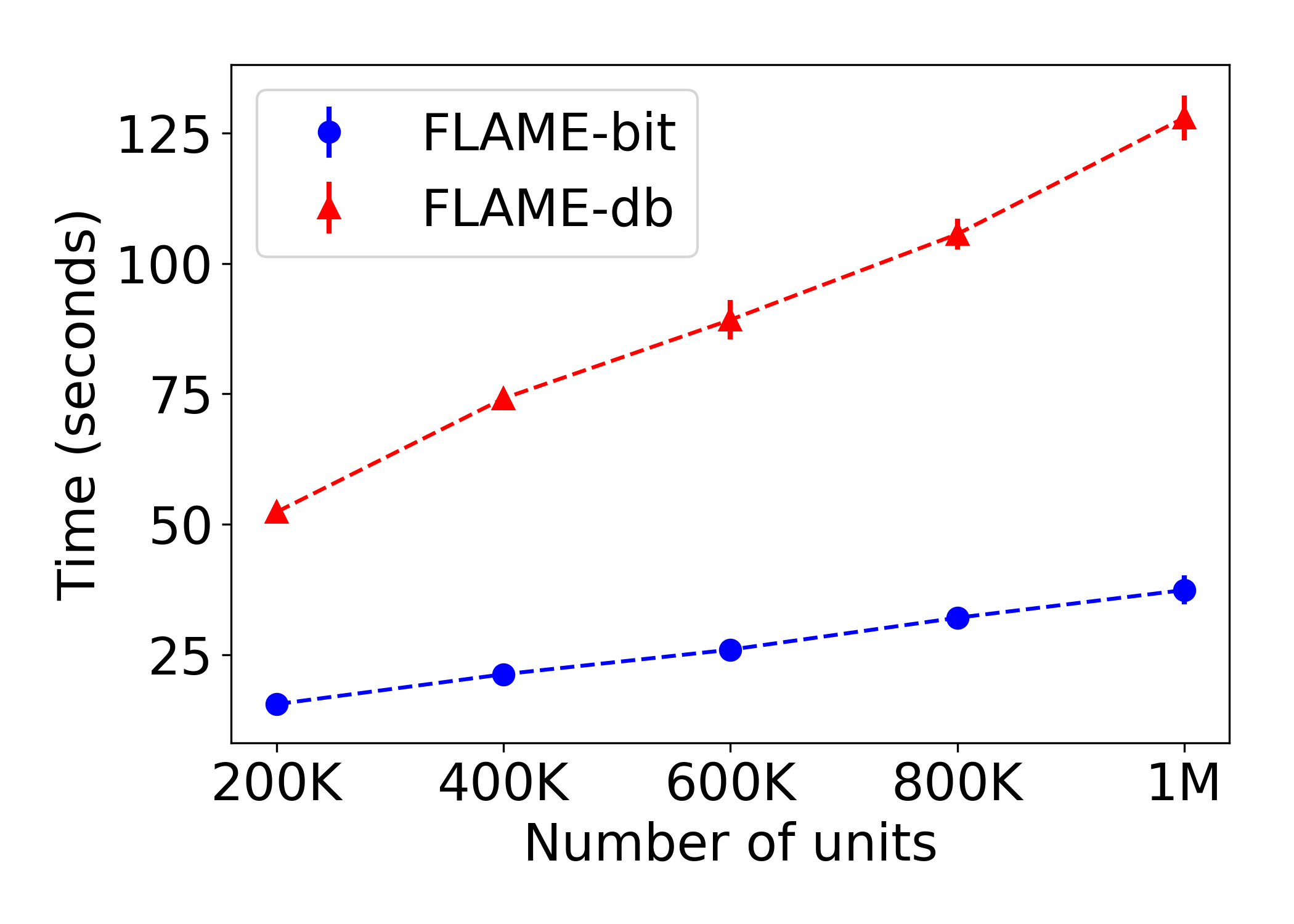}
     }
    \caption{
    This figure shows how the runtime scales with the number of units and number of covariates for FLAME-db and FLAME-bit. In general, FLAME-bit is faster when the number of covariates is relatively small so that the multiplication of the bit vectors is not too expensive. On the other hand, FLAME-db is faster when the number of variables is relatively large. In both subfigures, the size of holdout sets is fixed to 100,000 (50,000 control and 50,000 treated). \label{fig:timing} }
\end{figure}

\subsection{Effect of the $C$ parameter on the behavior FLAME \label{app:exp-tradeoff} }
This experiment aims to understand how the parameter $C$ would affect the estimation quality. Since the parameter $C$ trades off the Prediction Error ($\PE$) and Balancing Factor ($\BF$), we create a setting where variables that are more important to outcome prediction are less balanced. More specifically, we created 15,000 control units and 15,000 treated units with the outcomes given by
\begin{align}
y = \sum_{i=1}^{20} \frac{1}{i}  x_i + 10T + \epsilon \label{eq:tradeoff} 
\end{align}
where $x_i \sim \text{Bernoulli} \left( 0.1 + \frac{3(i-1)}{190} \right)$ for the control group and $x_i \sim \text{Bernoulli} \left( 0.9 - \frac{3(i-1)}{190} \right)$ for the treatment group, 
 $T \in \{ 0 , 1 \}$ is the treatment indicator, and $\epsilon \sim \mathcal{N} (0, 0.1)$. Based on (\ref{eq:criterion}), we expect that the larger the parameter $C$, the earlier the algorithm eliminates covariates of higher balancing factor. 

As we can see from Figure \ref{fig:tradeoff-matched} and \ref{fig:tradeoff-bubble}, larger $C$ values encourage FLAME to sacrifice some prediction quality in return for more matches; and vice versa. Better prediction quality leads to less biased estimates while a larger balancing factor leads to more matched units. This is a form of bias-variance tradeoff. 
Figure \ref{fig:tradeoff-combined} summarizes Figures \ref{fig:tradeoff-matched} and \ref{fig:tradeoff-bubble} by plotting the vertical axis of Figure  \ref{fig:tradeoff-matched} against the vertical axis of  Figure \ref{fig:tradeoff-bubble}. Since the percent of units matched in Figure \ref{fig:tradeoff-matched} is cumulative, the horizontal axis in Figure \ref{fig:tradeoff-combined} also (nonlinearly) corresponds to the dropping of covariates. 
In Figure \ref{fig:tradeoff-bubble}, each blue dot on the figures represents a matched group. The bias-variance tradeoff between estimation quality versus more matched groups is apparent; the left figure shows fewer but high quality (low bias) matches, whereas the right figure shows more matches that are lower quality (higher bias).

\begin{figure}[ht]
\centering   
    \subfloat[$C=0.1$ \label{subfig:fraction-Cp01}]{%
    \includegraphics[width=0.3\textwidth]{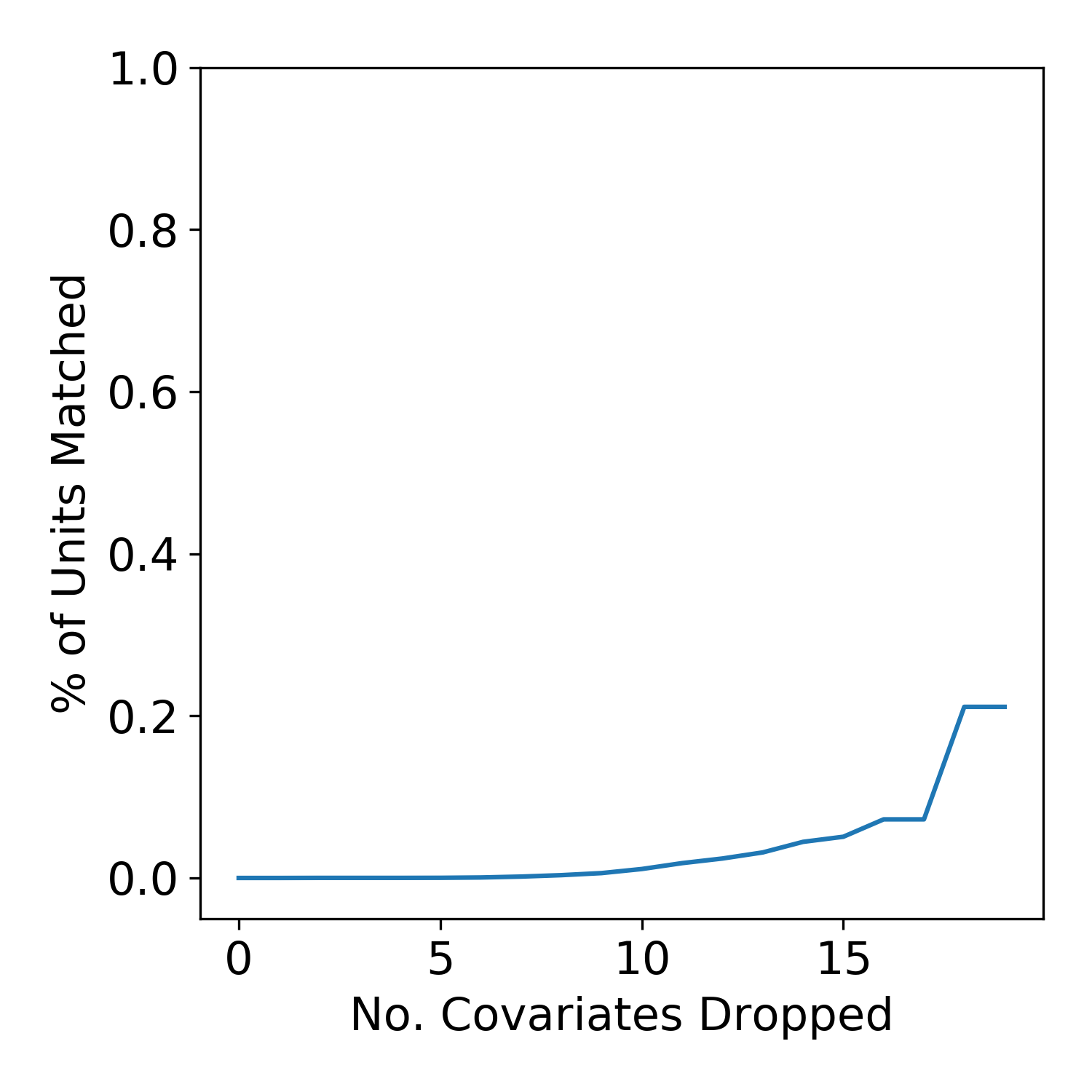} 
     } 
     \hspace{5pt}
     \subfloat[$C=0.5$ \label{subfig:fraction-Cp05}]{%
      \includegraphics[width=0.3\textwidth]{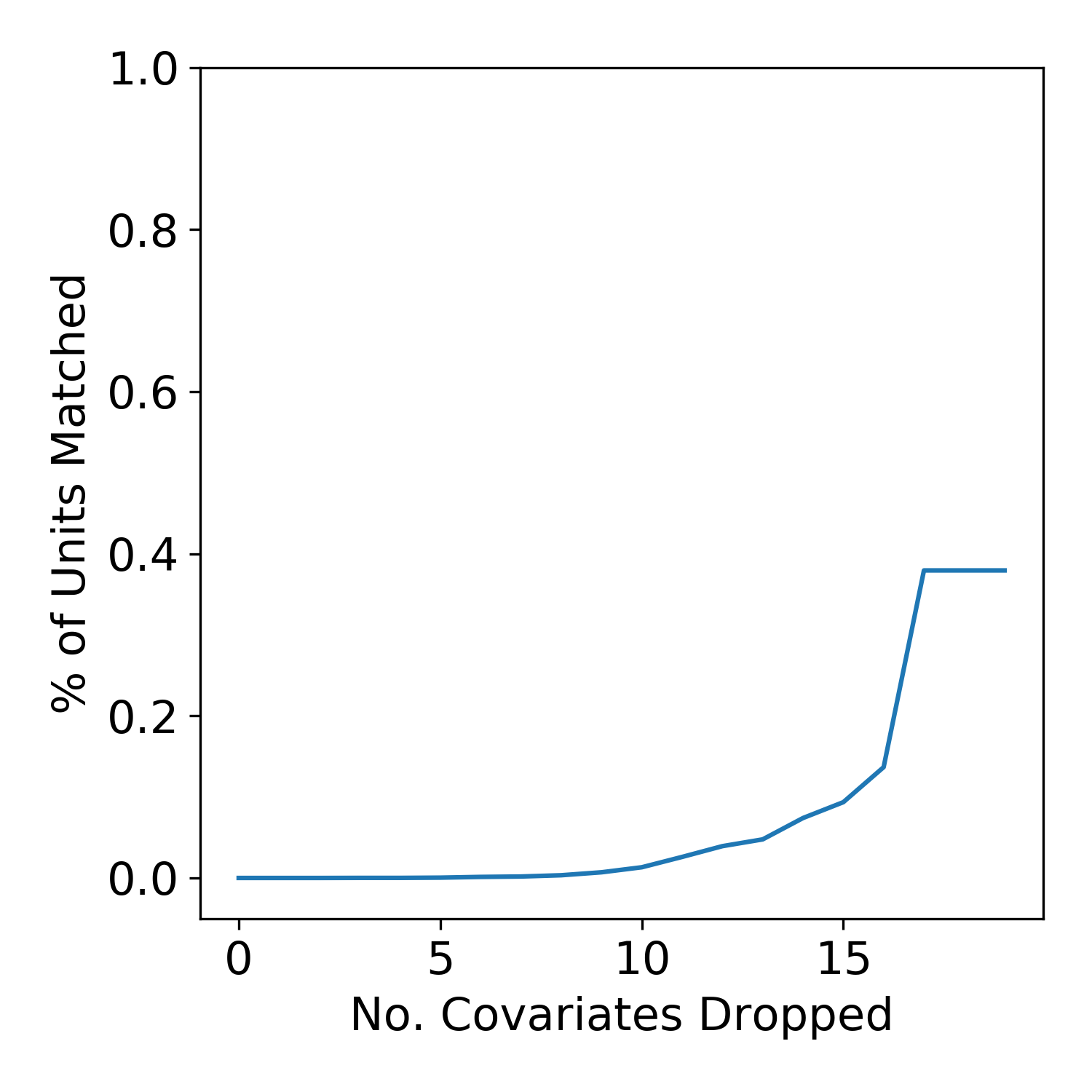}
     }
     \hspace{5pt}
     \subfloat[$C=1$ \label{subfig:fraction-C1}]{%
      \includegraphics[width=0.3\textwidth]{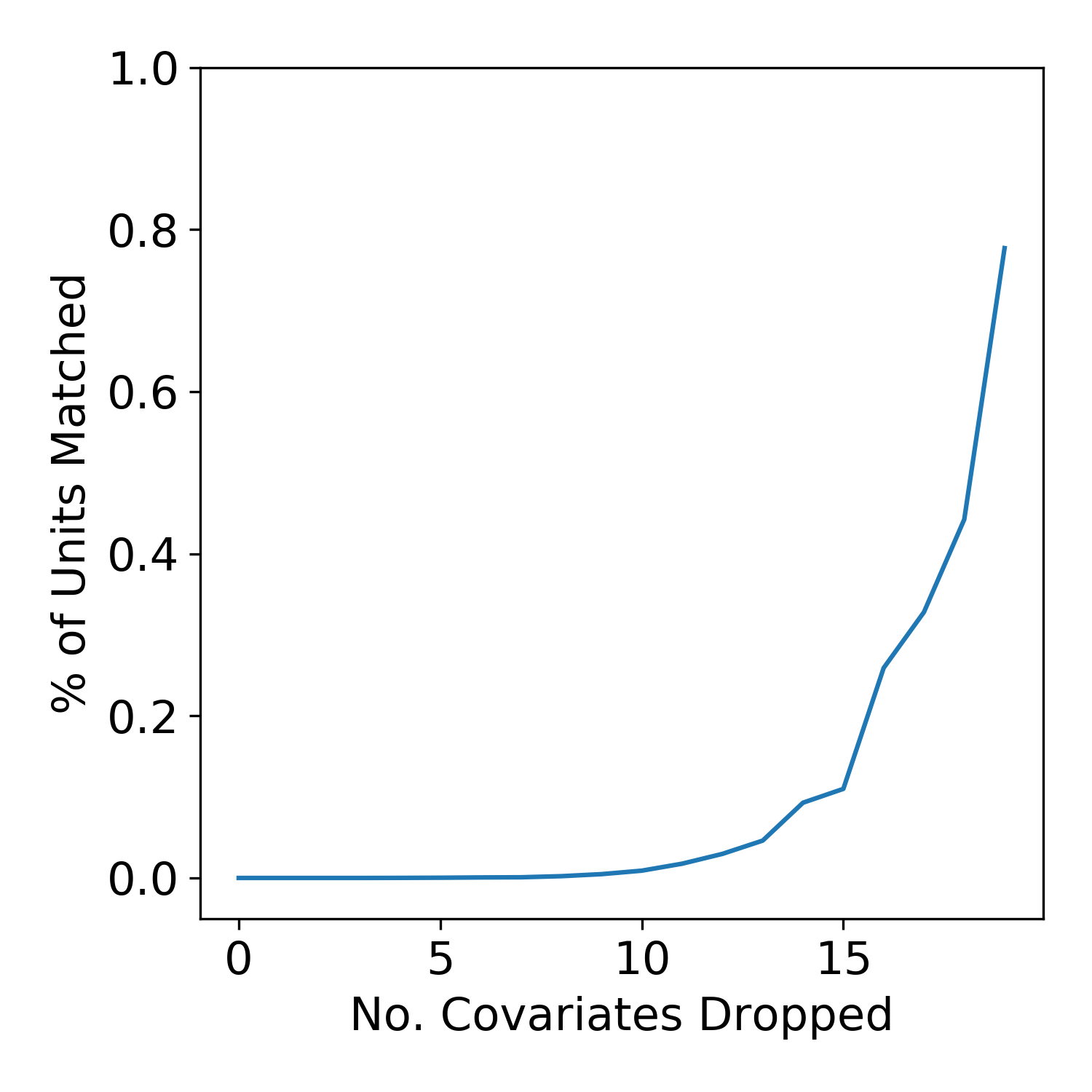}
     }
    \caption{This figure shows the percentage of units matched as FLAME eliminates variables with $C$ being 0.1, 0.5 and 1. More units are matched when the value of $C$ is large. The counts in these subfigures are cumulative. The vertical axis denotes the percentage of units matched.
    \label{fig:tradeoff-matched}}
\end{figure}

\begin{figure}[ht]
\centering   
    \subfloat[$C=0.1$ \label{subfig:bubble-Cp01}]{%
    \includegraphics[width=0.3\textwidth]{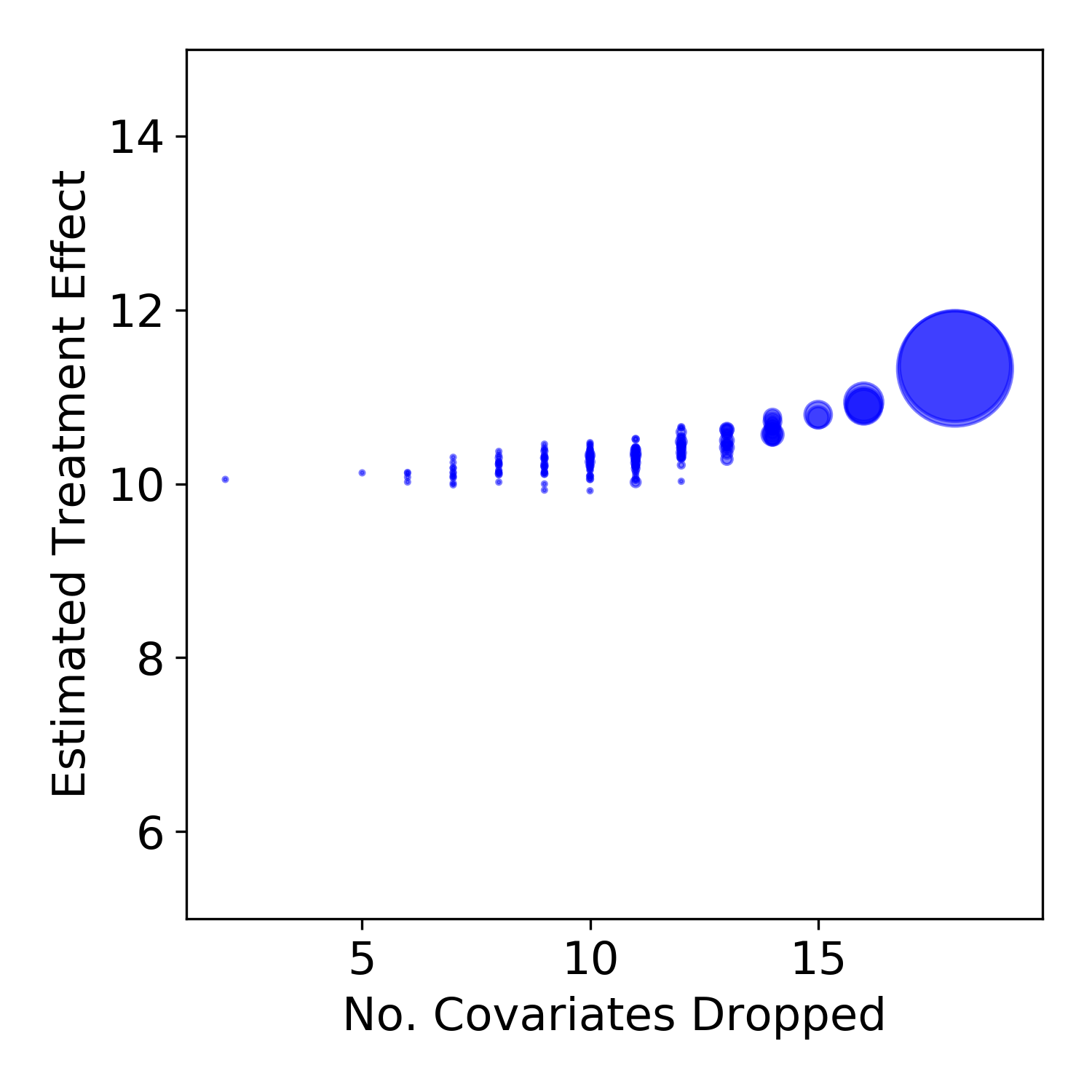}
     }
     \hspace{5pt}
     \subfloat[$C=0.5$ \label{subfig:bubble-Cp05}]{%
      \includegraphics[width=0.3\textwidth]{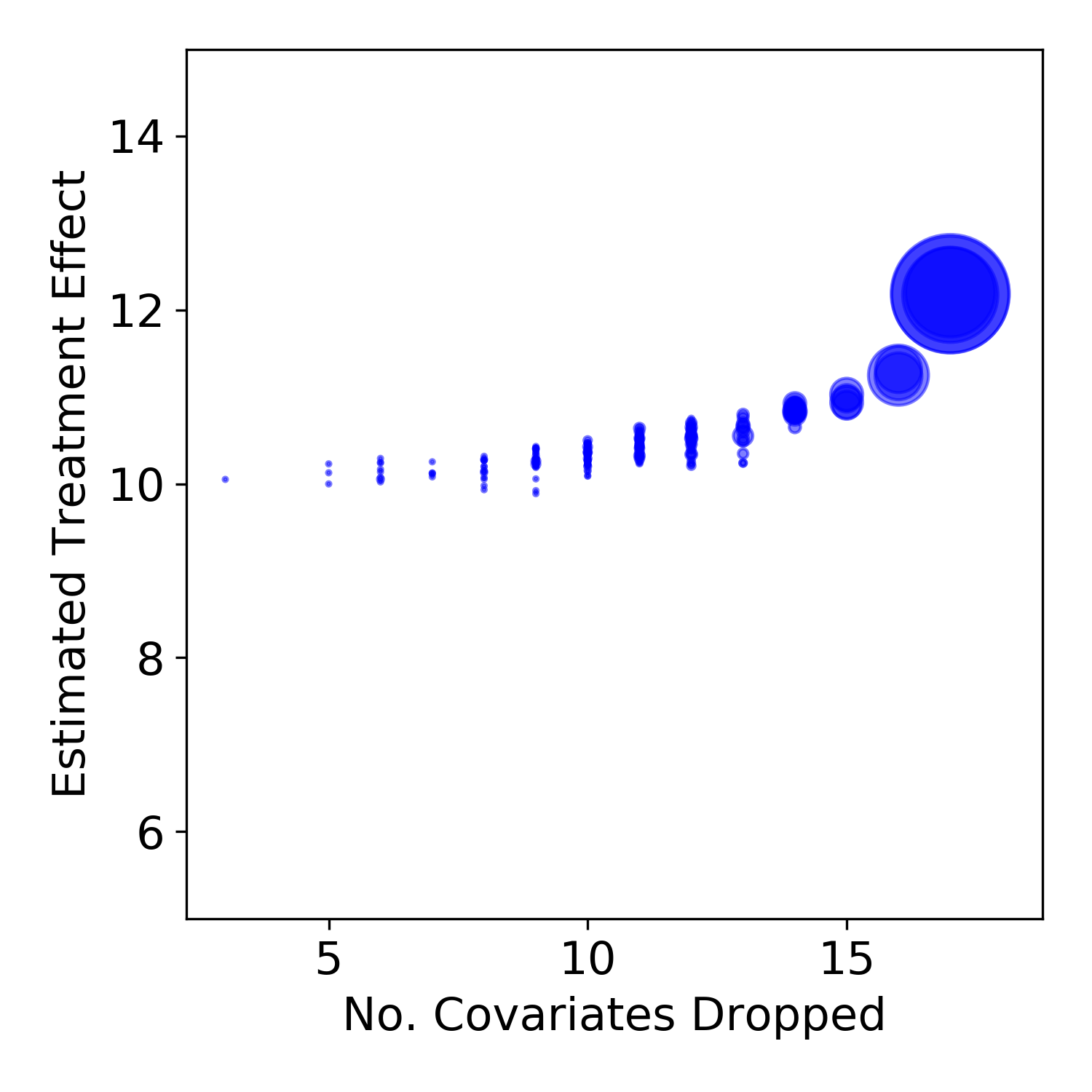}
     }
     \hspace{5pt}
     \subfloat[$C=1$ \label{subfig:bubble-C1}]{%
      \includegraphics[width=0.3\textwidth]{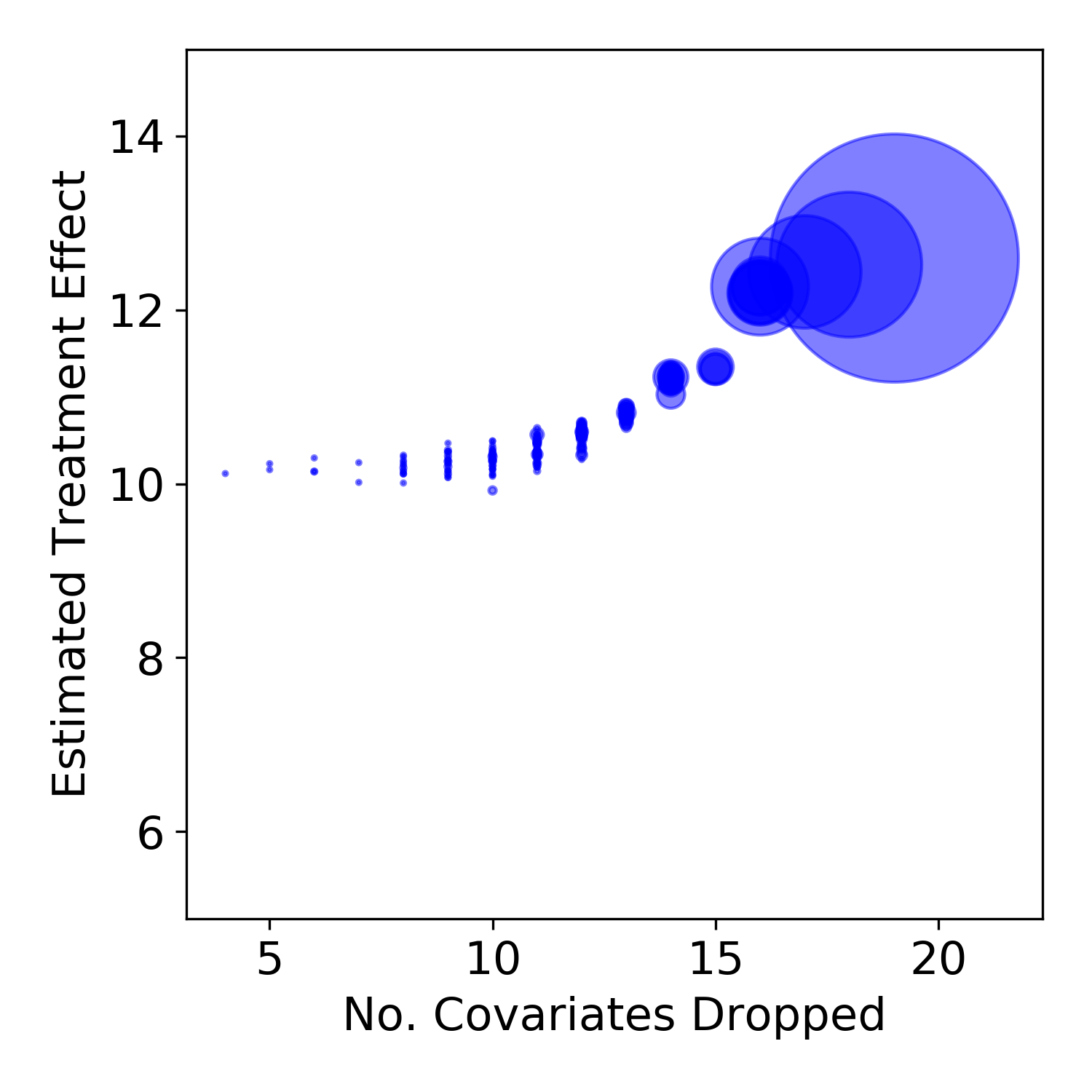}
     }
    \caption{This figure shows how estimation quality changes as FLAME eliminates variables with $C$ being 0.1, 0.5 and 1. The bias is smaller when the value of $C$ is small. Here the size of the dots represents the number of units corresponding to that dot.  
    \label{fig:tradeoff-bubble}}
\end{figure}

\begin{figure}[ht]
\centering   
    \subfloat[$C=0.1$ ]{%
    \includegraphics[width=0.3\textwidth]{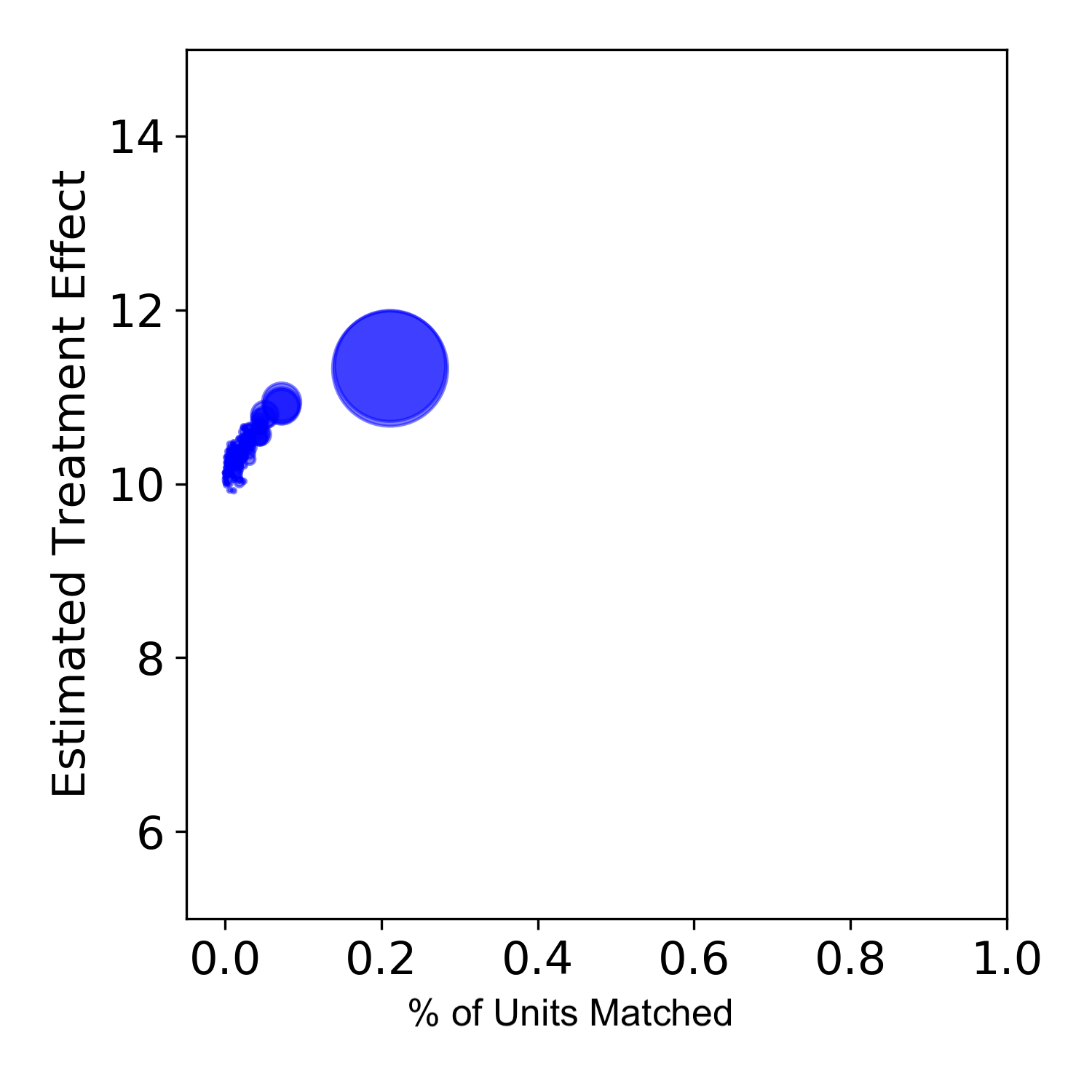}
     }
     \hspace{5pt}
     \subfloat[$C=0.5$ ]{%
      \includegraphics[width=0.3\textwidth]{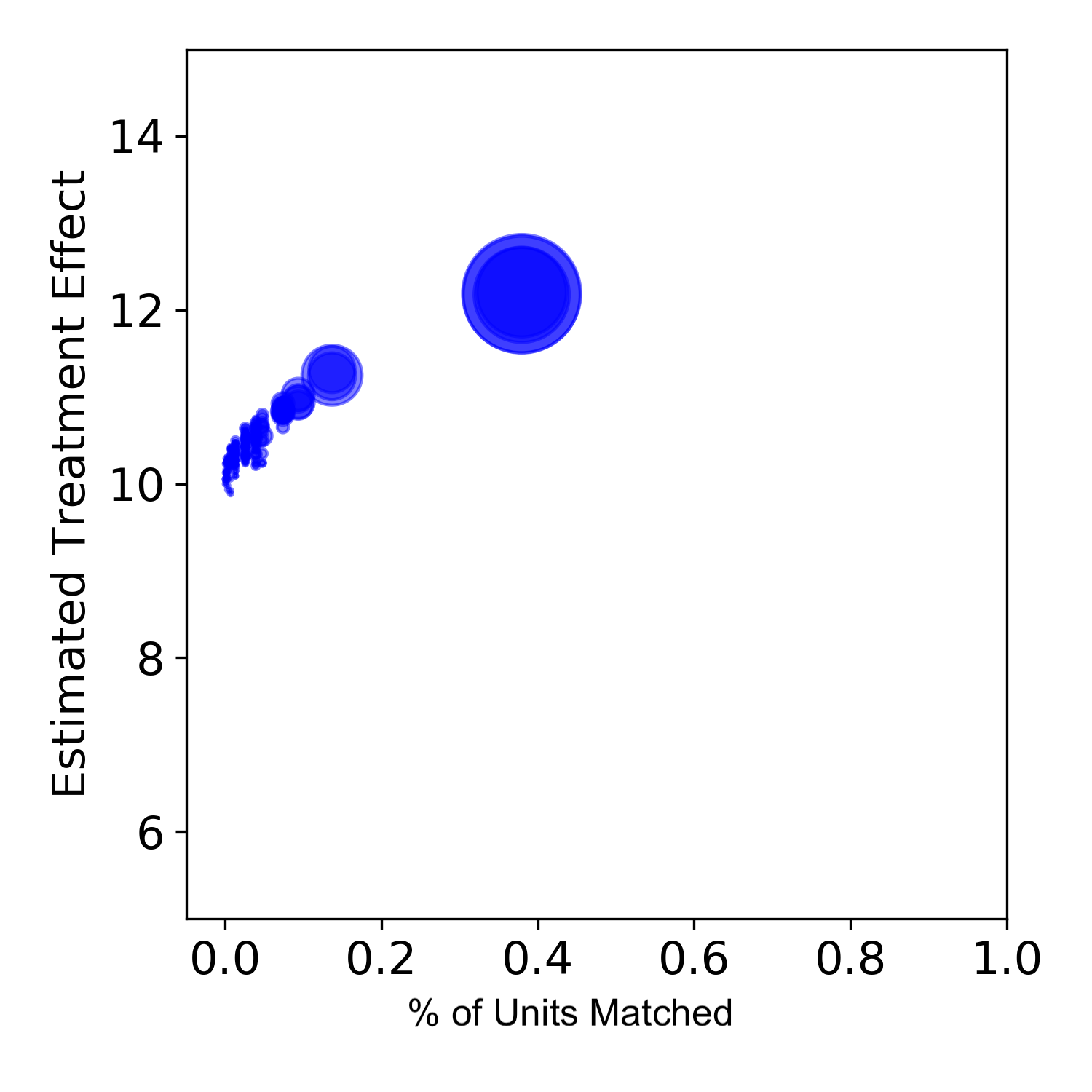}
     }
     \hspace{5pt}
     \subfloat[$C=1$ ]{%
      \includegraphics[width=0.3\textwidth]{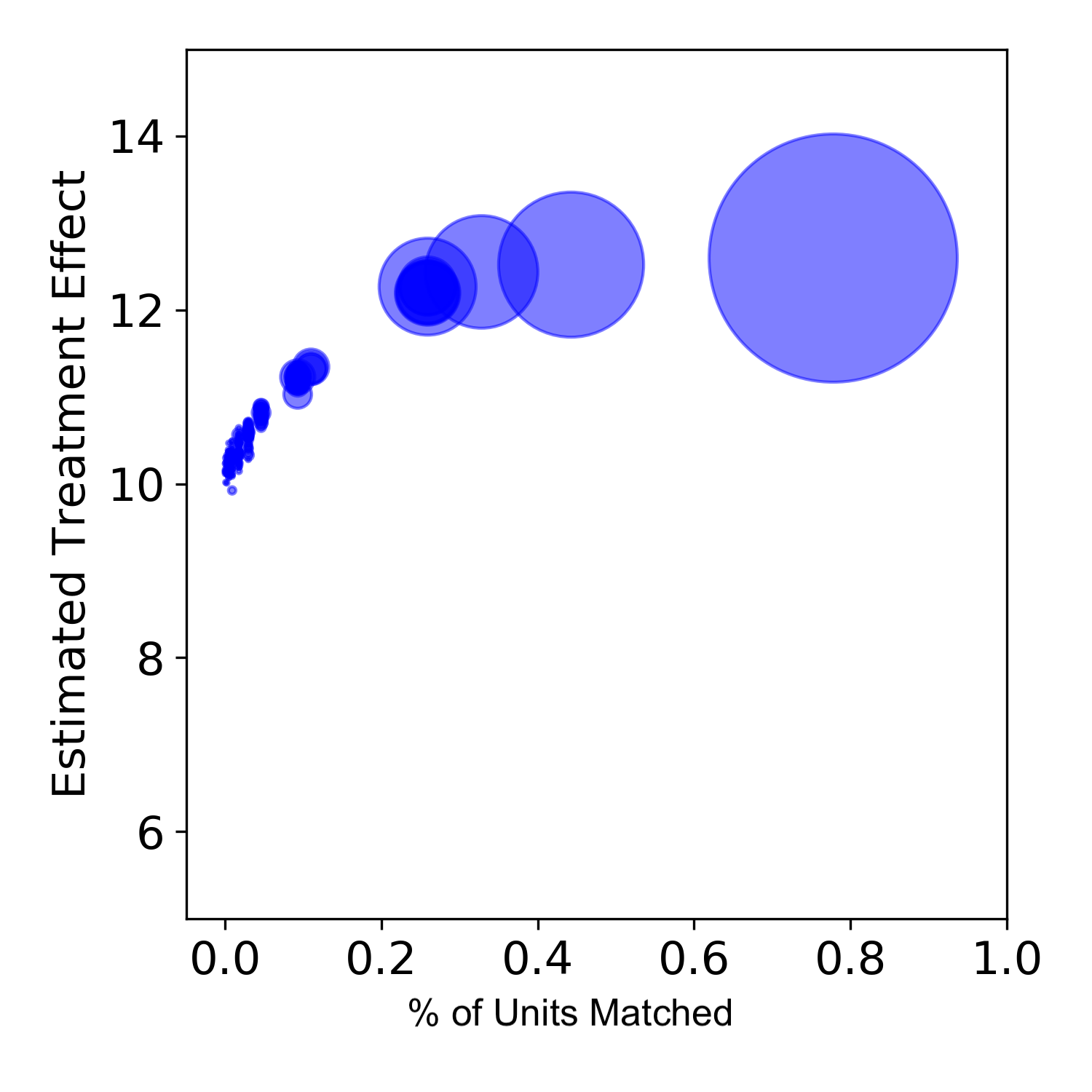}
     }
    \caption{Estimated treatment effect versus \% of units matched. 
    This figure is a summary of Figure \ref{fig:tradeoff-matched} and Figure \ref{fig:tradeoff-bubble}. 
    Here the horizontal axis is the vertical axis of Figure  \ref{fig:tradeoff-matched} and the vertical axis is the vertical axis of Figure \ref{fig:tradeoff-bubble}. Size of the dots represents the number of units corresponding to that dot.  
    }
    \label{fig:tradeoff-combined}
\end{figure}

\subsection{Decay of performance as less important variables are eliminated} 
To better understand the behavior of FLAME as it eliminates variables, we create a setting where the variables are all assigned non-zero importance. In this case, 20 covariates are used. As FLAME drops variables, prediction quality smoothly degrades. This is meant to represent problems where the importance of the variables decreases according to an exponential or a power-law, as is arguably true in realistic settings. Accordingly, we create 10,000 control units and 10,000 treated units with the outcomes generated as follows: 
\begin{align}
y = \sum_{i=1}^{20} \alpha_i x_i + 10 T + \epsilon, \label{eq:synthetic-data-bubble}
\end{align}
where $T \in \{ 0, 1\}$ is the binary treatment indicator, $x_i \sim \text{Bernoulli}(0.5)$, $\epsilon \sim \mathcal{N} (0, 0.1)$, $\alpha_i = 5 \times \left( \frac{1}{2} \right)^i$ for exponential decay (in Figure \ref{subfig:bubble-exp}),
and $\alpha_i = 5 \times \frac{1}{i} $ for power-law decay (in Figure \ref{subfig:bubble-power}). For both exponential and power law decay, variables with smaller indices are more important, and variables with higher indices are less important. 
In (\ref{eq:synthetic-data-bubble}), all of the covariates contribute to the outcome positively. In the real world, variables can contribute to the outcome either positively or negatively, which leads to a smaller estimation bias. This is because 
when positive covariates and negative covariates are dropped in mixed order, some bias cancels out. 
The case we are considering, where all $\alpha_i$'s are positive, is essentially a worst case. It allows us to see more easily how prediction quality degrades.

The results from FLAME are shown in Figure \ref{fig:bubble}, which shows that (i) the variability of estimation degrades smoothly as more variables are dropped, and (ii) the bias still remains relatively small. In this experiment, the value of $C$ is set to be $0.001$. The effect of varying the parameter $C$ is studied in Appendix  \ref{app:exp-tradeoff}. 
Since $C$ is set to be small, FLAME dropped the covariates in ascending order of importance. 

\begin{figure}
  \centering
  \subfloat[Exponential decay\label{subfig:bubble-exp}]{%
       \includegraphics[width = 0.4\textwidth]{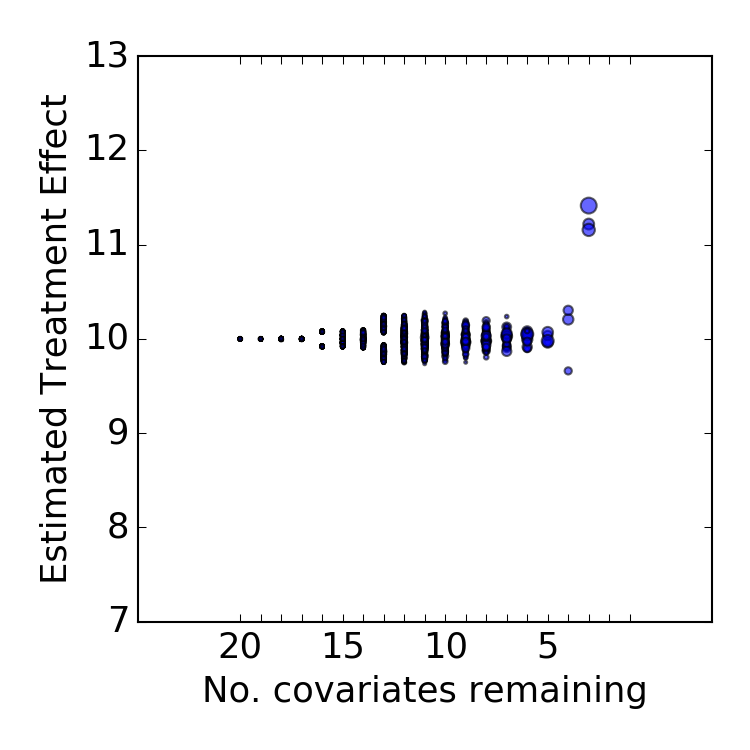}
     }
     \hspace{10pt}
     \subfloat[Power law decay\label{subfig:bubble-power}]{%
       \includegraphics[width = 0.4\textwidth]{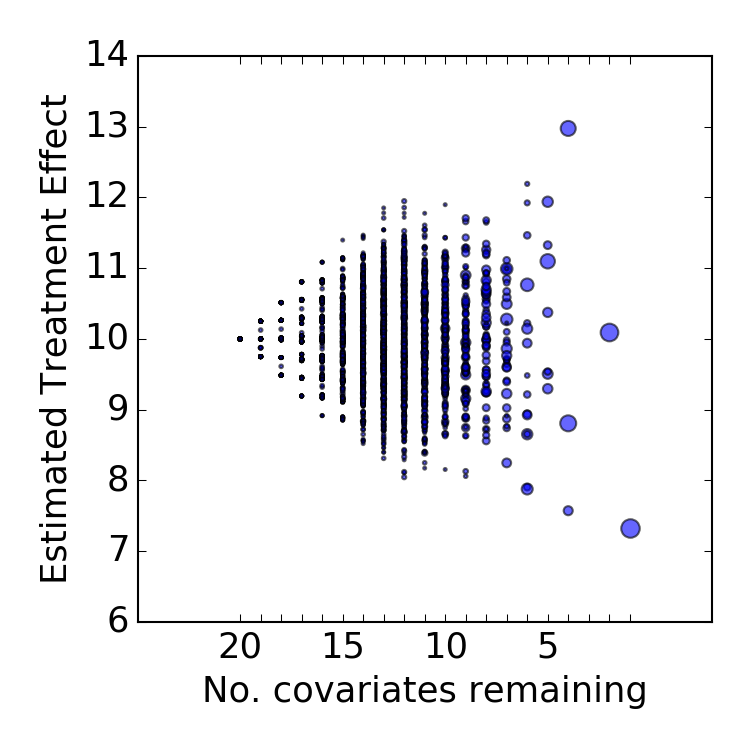}
     }
    \caption{\small Degradation of treatment effect estimates as FLAME drops variables. \textit{Left:} variable importance decreases exponentially with base being $\frac{1}{2}$. \textit{Right:} the variable importance decreases according to a power-law with exponent -1. The true treatment effect is 10 for all units in both subfigures. Dot size represents the number of units corresponding to that dot. There are fewer matches in high dimensions, and those matches are of higher quality. This shows that the bias of estimation degrades smoothly as we eliminate variables. 
    }
    \label{fig:bubble} 
\end{figure} 

\subsection{Natality Data Additional Details}
\subsubsection{Preprocessing for Natality Data}
\label{secnatality:preprocessing}

\textit{\textit{Different versions of birth certificate.~}}
As of 2010, two versions of birth certificates are used in the United States: the 1989 Revision of the U.S. Standard Certificate of Live Birth (unrevised version) and the 2003 revision of the U.S. Standard Certificate of Live Birth (revised version). The two versions record different variables about the parents and the infant. We therefore focus only on the data records of the revised version. 

\textit{\textit{Multiple columns coding the same information.~}}
In the raw data, several columns encode the same information when a numeric variable has a large range. For a few columns encoding the same information, one of them contains the raw values of the variable, and the rest are typically discretizations to different coarsening scales.   

\subsubsection{Covariates Dropping Order}
\label{app:natality-cov-tables}

The order of covariate dropping for FLAME on the natality data is summarized in Tables \ref{tab:drop-order-nicu} and \ref{tab:drop-order}. 

\begin{table}[h!]
    \centering
    \caption{Covariates dropping order of a FLAME run on the natality data when the outcome is NICU admission. Covariates dropped earlier are listed closer to the top. In the table, we use \textit{the Manual} to refer to the Natality Data Manual \citep{natality2010}. 
    \label{tab:drop-order-nicu}} 
    \resizebox{\textwidth}{!} 
    { 
    \begin{tabular}{r|l} 
        \textbf{covariate} & \textbf{additional remarks}  \\ \hline \hline 
        
        Father's Age & Father's Age Recode 11 column in the manual. \\
        Weekday & Weekday column in the manual. \\ 
        Mother's Education & Mother's Education column in the manual. \\
        Total Birth Order & Total Birth Order Recode column in the manual. \\ 
        Live Birth Order & Live Birth Order Recode column in the manual. \\ 
        Father's Race & Father's Bridged Race column in the manual. \\ 
        Mother's Race & Mother's Bridged Race column in the manual. \\ 
        Mother's Age & Mother's Age Recode 9 column in the manual. \\ 
        Chronic Hypertension & Chronic Hypertension column in the manual. \\ 
        Previous Cesarean Deliveries & Previous Cesarean Deliveries column in the manual. \\ 
        Mother's Marital Status & Mother's Marital Status column in the manual. \\ 
        Prepregnancy Diabetes & Prepregnancy Diabetes column in the manual. \\ 
        Prepregnancy Hypertension & Prepregnancy Hypertension column in the manual. \\ 
        Previous Preterm Birth & Previous Preterm Birth column in the manual. \\ 
        Birth Place & Birth Place Recode column in the manual. \\
        Sex of Infant & Sex of Infant column in the manual. \\ 
        \hline
        
        \end{tabular} 
    } 
\end{table} 

\begin{table}[h!]
    \centering
    \caption{Covariates dropping order of a FLAME run on the natality data when the outcome is infant's birth weight. Covariates dropped earlier are listed closer to the top. In the table, we use \textit{the Manual} to refer to the Natality Data Manual \citep{natality2010}. 
    \label{tab:drop-order}} 
    \resizebox{\textwidth}{!} 
    { 
    \begin{tabular}{r|l} 
        \textbf{covariate} & \textbf{additional remarks}  \\ \hline \hline
        
        Father's Age & Father's Age Recode 11 column in the manual. \\
        Mother's Education & Mother's Education column in the manual. \\  
        Weekday & Weekday column in the manual. \\ 
        Mother's Age & Mother's Age Recode 9 column in the manual. \\ 
        Father's Race & Father's Bridged Race column in the manual. \\ 
        Mother's Marital Status & Mother's Marital Status column in the manual. \\ 
        Prepregnancy Diabetes & Prepregnancy Diabetes column in the manual. \\ 
        Previous Cesarean Deliveries & Previous Cesarean Deliveries column in the manual. \\ 
        Total Birth Order & Total Birth Order Recode column in the manual. \\ 
        Prepregnancy Hypertension & Prepregnancy Hypertension column in the manual. \\ 
        Chronic Hypertension & Chronic Hypertension column in the manual. \\ 
        Live Birth Order & Live Birth Order Recode column in the manual. \\ 
        Previous Preterm Birth & Previous Preterm Birth column in the manual. \\ 
        Sex of Infant & Sex of Infant column in the manual. \\ 
        Mother's Race & Mother's Bridged Race column in the manual. \\ 
        Birth Place & Birth Place Recode column in the manual. \\
        \hline
        \end{tabular} 
    } 
\end{table}





\subsubsection{Evaluating NICU admissions}
This subsection includes a figure of the matched group size versus estimated odds ratio of NICU admission due to smoking. See Figure \ref{fig:natality-preterm}. 
\begin{figure}[h]
    \centering
    \includegraphics[width = 0.5\textwidth]{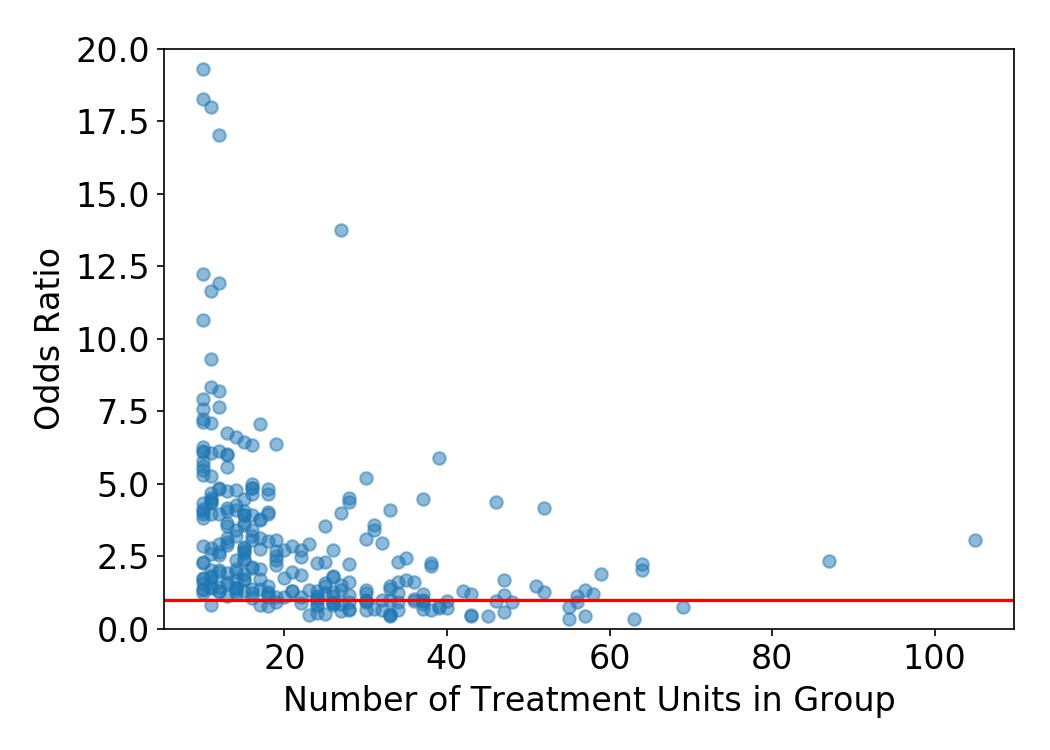}
    \caption{Scatter plot of odds ratios versus matched group size that is used to diagnose data quality issues for granular causal inference. The red line is $y = 1$. Only matched groups with more than 30 control units and 10 treatment units are shown. The overall odds ratio (weighted by group sizes) is approximately 2.67, which indicates a harmful outcome resulting from smoking. Our plot shows that in some cases the odds ratio is smaller than 1, even if the group size is relatively large. This variance calls for more investigation on the effect of smoking on NICU admission, especially from a medical perspective. 
    \label{fig:natality-preterm} 
    }
\end{figure} 

\subsection{US Census Data Additional Details}
\label{app:census-detail}

For preprocessing, we first discard the units whose income is zero (most of these units have age smaller than 16). This gives us 1,220,882 units to work with. Using 59 of the variables ($p = 59$) and these units, we evaluate scalability of different methods (Table \ref{tab:compare-timing-census}). More specifically, the experiment is posited as studying the causal effect of \emph{marital status} ($T$) on \emph{wage or salary income} ($Y$) in the year 1989, with the assumptions stated in the introduction (e.g., no unmeasured confounding). The wage variable (dIncome1) is binarized in the following way: 0 for $( 0, 15,000]$ (low income), and 1 for larger than 15,000 (high income). This binarization gives 564,755 low-income people and 656,127 high-income people. The marital state is binarized with 1 for being married and 0 for unmarried (including divorced, widowed and separated). This preprocessing gives 722,688 treated units (married) and 498,194 control units (unmarried). We randomly sampled 10\% of these units (122,089 units) as the training set. As a real dataset with a large number of rows and columns, the US Census Data is mainly used for scalability evaluation. For the estimated treatment effect, in agreement with the literature \citep{zagorsky2005marriage}, FLAME suggests marriage contributes positively to income. 

  \par
\textbf{Running time comparison:} Methods in Table \ref{tab:compare-timing-census} other than FLAME-db either did not finish within 10 hours or crashed, whereas FLAME-db took 1.33 hours. In this case, FLAME-bit encounters a overflow problem. 

\section{Causal Forest is Not a Matching Method}

The causal forest algorithm \citep{wager2018estimation, athey2019generalized} is a method that learns both a model for the counterfactuals and one for the propensity. While causal forests are used to estimate CATEs, the method should not be interpreted as a matching method. Matching methods match a small subset of units to estimate CATEs. In practice, matching methods are essentially subgroup analysis methods. 
One top desiderata of a matching method is that one unit should not be matched to too many other units. If a unit is matched to too many others, then the match will cease to be interpretable, which does not easily permit case-based reasoning. 

In order to interpret causal forests as a matching method, one would need to define a match to be the number of points sharing at least one leaf of one tree in the forest with a given unit. However, with causal forests, this number increases rapidly with the number of trees in the forest. For instance, using the data generation process from  Section \ref{sec:experiments} (repeated below), if one runs causal forest for 1000 iterations, the average number of units matched to any given unit is 450. After 5000 iterations, the average number of units matched to a given unit is over 1000. Let us demonstrate this: 

We use the same model as in Section \ref{sec:experiments} to generate synthetic data, for each unit $i$:
\begin{align*}
    y_i = \sum_{j}^{p} \alpha_j x_{ij} + T \sum_{j=1}^{p} \beta_j x_{ij} + T_i\cdot U \sum_{j,\gamma,\gamma>j}  x_{ij} x_{i\gamma}, 
\end{align*}
where $i$ indexes units and $j$ indexes covariates, $\alpha_j$ is a baseline effect of each covariate, $\beta_j$ is the treatment effect conditional on covariate $j$. The model also includes a nonlinear interaction effect between covariates: $x_jx_{\gamma}$, where each covariate interacts with all the covariates that succeeds it. This term is weighted by the coefficient $U$.
We simulate data from a model with 10,000 treatment and 10,000 control units, and 30 covariates, with 10 important and 20 unimportant (i.e., such that $\alpha_j, \beta_j = 0$). We implement causal forests with 200 trees and 1000 training samples per tree. 

The result is in Figure \ref{fig:treesvsnumber}, illustrating how the size of matched groups grows as a function of the number of iterations of the algorithm.

\begin{figure}[!htbp]
    \centering
    \includegraphics[width=0.7\linewidth]{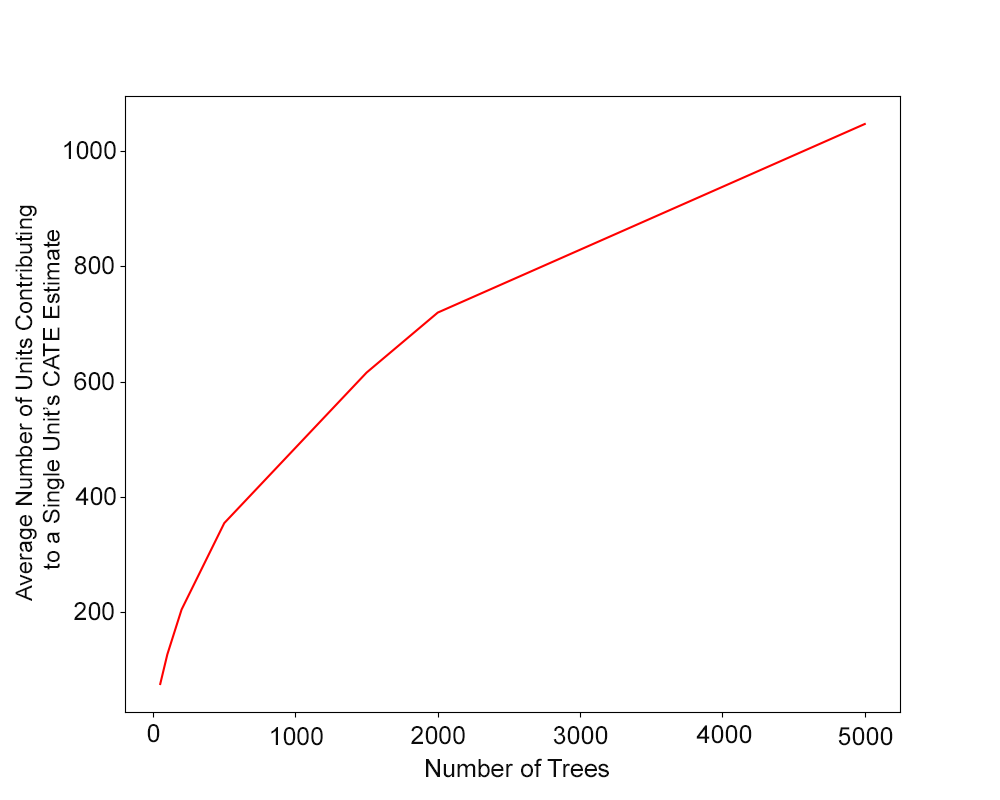}
    \caption{Average number of units used to predict a single unit's CATE by causal forests as the number of trees increases. }
    \label{fig:treesvsnumber}
\end{figure}

Thus, causal forests cannot be interpreted as a matching method for the following reasons:
\begin{itemize}
    \item The size of the matched groups depends on the runtime of the method. Thus it is not clear where one would stop to define it as a matching method or as a non-matching method.
    \item If one were to try to obtain interpretable matched groups from causal forest, one would need to stop after only a few trees, which is not recommended by \citet{wager2018estimation, athey2019generalized} because it would hurt performance.
    \item Causal forest was not designed to be a matching method. Examining its large matched groups would not necessarily achieve the main goals of case-based reasoning, such as troubleshooting variable quality or investigating the possibility of missing confounders.
\end{itemize}

\end{document}